\documentclass[a4paper,UKenglish,cleveref, autoref, thm-restate]{lipics-v2021}
\pdfoutput=1 %uncomment to ensure pdflatex processing (mandatatory e.g. to submit to arXiv)
\hideLIPIcs  %uncomment to remove references to LIPIcs series (logo, DOI, ...), e.g. when preparing a pre-final version to be uploaded to arXiv or another public repository

\usepackage{xspace}
\usepackage{tikz}
\usepackage{mathtools}
\usepackage{thm-restate}
\usepackage{graphicx}
\usepackage{lmodern}
\graphicspath{ {./pictures/} }

\newcommand{\eqrefleft}[1]{(\ref{#1}, left)}
\newcommand{\eqrefright}[1]{(\ref{#1}, right)}

\newcommand{\p}{\varphi}

\newcommand{\mdl}{\models}

\newcommand{\sbs}{\subseteq}

\newcommand{\lp}{{\langle}}
\newcommand{\rp}{{\rangle}}

\newcommand{\impd}{\leftarrow}

\label{key}

\newcommand{\expand}{\textit{expand}}
\newcommand{\accept}{\textit{accept}}

\newcommand{\notaccept}{\textit{notaccept}}

\newcommand{\Expand}{\ensuremath{\textit{Expand}_\preccurlyeq}\xspace}
\newcommand{\Accept}{\textit{Accept}}
\newcommand{\Notaccept}{\textit{NotAccept}}

\newcommand{\expandom}{\ensuremath{\Expand}\xspace}
\newcommand{\acceptom}{{\ensuremath{\Accept}\xspace}}
\newcommand{\notacceptom}{\ensuremath{\Notaccept}\xspace}

\newcommand{\Apic}{\ensuremath{\A_{(\pi, G)}}\xspace}

\newcommand{\Apig}{\ensuremath{\A_{(\pi, G)}}\xspace}

\newcommand{\avec}[1]{\boldsymbol{#1}}

\newcommand{\ovl}[1]{\overline{#1}}

\newcommand{\E}{\ensuremath{\mathcal{E}}\xspace}

\newcommand{\IDBpi}{\ensuremath{\textit{IDB}(\pi)}\xspace}
\newcommand{\EDBpi}{\ensuremath{\textit{EDB}(\pi)}\xspace}

\newcommand{\Tline}[1]{\ensuremath{\textit{Tline}(#1)}\xspace}

\newcommand{\poly}[1]{{\mathrm{poly}\left(#1\right)}}

\newcommand{\dom}{\textup{dom}\xspace}

\newcommand{\Db}{\ensuremath{\mathcal{D}}\xspace}

\newcommand{\height}[1]{\ensuremath{\textup{height}(#1)}\space}
\newcommand{\wid}[1]{\ensuremath{\textup{width}(#1)}\space}

\newcommand{\obj}{\mathrm{obj}}

\newcommand{\tem}{\mathrm{tem}}
\newcommand{\Circ}{\raisebox{0.25ex}{\text{\scriptsize{$\bigcirc$}}}}
\newcommand{\Next}{{\ensuremath{\Circ}}\xspace}
\newcommand{\Prev}{{\ensuremath{\Circ^{-}}}\xspace}

\newcommand{\Diam}{\raisebox{0.1ex}{\text{\small{$\Diamond$}}}}
\newcommand{\D}{\ensuremath{\Diam}\xspace}
\newcommand{\Df}{\ensuremath{\D}\xspace}
\newcommand{\Dp}{\ensuremath{\D\!^{-}}\xspace}

\newcommand{\Drefl}{\ensuremath{\D^*}\xspace}

\newcommand{\monodic}{\ensuremath{\mathop{\ooalign{$\Box$ \cr \kern0.57ex \raisebox{0.2ex}{\scalebox{0.55}{$1$}}}\rule{0pt}{1.5ex} \kern-0.7ex}}\xspace}

\newcommand{\Nextone}{\ensuremath{\mathop{\ooalign{$\Next$ \cr \kern0.57ex \raisebox{0.3ex}{\scalebox{0.55}{$1$}}}\rule{0pt}{1.5ex} \kern-0.7ex}}\xspace}

\newcommand{\Wnextone}{\ensuremath{\mathop{\ooalign{$\Wnext$ \cr \kern0.57ex \raisebox{0.2ex}{\scalebox{0.55}{\textcolor{white}{$1$}}}}\rule{0pt}{1ex} \kern-0.7ex}}\xspace}

\newcommand{\Xnext}{^{\,\smash{\raisebox{0pt}{$\scriptscriptstyle\bigcirc$}}}}

\newcommand{\Xall}{^{\,\scriptscriptstyle\bigcirc {\smash{\raisebox{-1.3pt}{$\diamond$}}}}}

\newcommand{\Xdiamond}{^{\,\smash{\raisebox{-1pt}{$\diamond$}}}}

\newcommand{\dlnd}{\ensuremath{\textsl{datalog}\Xall}\xspace}

\newcommand{\mdlnd}{\ensuremath{\textsl{datalog}_m\Xall}\xspace}
\newcommand{\mdln}{\ensuremath{\textsl{datalog}_m\Xnext}\xspace}

\newcommand{\lmdlnd}{\ensuremath{\textsl{datalog}_{\it lm}\Xall}\xspace}
\newcommand{\lmdln}{\ensuremath{\textsl{datalog}_{\it lm}\Xnext}\xspace}
\newcommand{\lmdld}{\ensuremath{\textsl{datalog}_{\it lm}\Xdiamond}\xspace}

\newcommand{\LTL}{\ensuremath{\textsl{L\!TL}}\xspace}

\newcommand{\FO}{\textup{FO}\xspace}

\newcommand{\FOLess}{\ensuremath{\FO(<)}\xspace}

\newcommand{\MTDLogX}{\mdln}

\newcommand{\ACz}{{\ensuremath{\textsc{AC}^{0}}}\xspace}
\newcommand{\ACCz}{{\ensuremath{\textsc{ACC}^{0}}}\xspace}
\newcommand{\NCo}{{\ensuremath{\textsc{NC}^{1}}}\xspace}
\newcommand{\LogSpace}{\textsc{L}\xspace}
\newcommand{\NLogSpace}{\textsc{NL}\xspace}
\newcommand{\PTime}{\textsc{P}\xspace}

\newcommand{\PSpace}{\textsc{PSpace}\xspace}
\newcommand{\NPSpace}{\textsc{NPSpace}\xspace}
\newcommand{\ExpTime}{\textsc{ExpTime}\xspace}

\newcommand{\ExpSpace}{\textsc{ExpSpace}\xspace}

\newcommand{\Smf}{\ensuremath{\mathfrak{S}}\xspace}

\newcommand{\Emc}{\ensuremath{\mathcal{E}}\xspace}

\newcommand{\Gmc}{\ensuremath{\mathcal{G}}\xspace}

\newcommand{\Omc}{\ensuremath{\mathcal{O}}\xspace}

\newcommand{\Rmc}{\ensuremath{\mathcal{R}}\xspace}
\newcommand{\Smc}{\ensuremath{\mathcal{S}}\xspace}

\newcommand{\Ubf}{\ensuremath{\mathbf{U}}\xspace}

\newcommand{\Z}{\mathbb{Z}}
\newcommand{\N}{\mathbb{N}} 
\newcommand{\Nbb}{\ensuremath{\mathbb{N}}\xspace}
\newcommand{\Zbb}{\ensuremath{\mathbb{Z}}\xspace}

\newcommand{\Uol}{\ensuremath{\overline{U}}\xspace}

\newcommand{\Xol}{\ensuremath{\overline{X}}\xspace}
\newcommand{\Yol}{\ensuremath{\overline{Y}}\xspace}
\newcommand{\Zol}{\ensuremath{\overline{Z}}\xspace}

\newcommand{\A}{\ensuremath{\mathcal{A}}\xspace}

\newcommand{\sgn}{\mathrm{sgn}}

\title{On Deciding the Data Complexity of Answering Linear Monadic
  Datalog Queries with LTL Operators (Extended Version)}

\titlerunning{Deciding the Data Complexity of LTL Monadic Datalog Queries}

\author{Alessandro Artale}{Faculty of Engineering, Free University of Bozen-Bolzano, piazza Domenicani 3, Bozen-Bolzano 39100, Italy}{artale@inf.unibz.it}{https://orcid.org/0000-0002-3852-9351}{}

\author{Anton Gnatenko}{Faculty of Engineering, Free University of Bozen-Bolzano, piazza Domenicani 3, Bozen-Bolzano 39100, Italy}{anton.gnatenko@student.unibz.it}{https://orcid.org/0000-0003-1499-2090}{}

\author{Vladislav Ryzhikov}{Birkbeck, University of London, Malet Street, London WC1E 7HX, UK}{v.ryzhikov@bbk.ac.uk}{https://orcid.org/0000-0002-6847-6465}{}

\author{Michael Zakharyaschev}{Birkbeck, University of London, Malet Street, London WC1E 7HX, UK}{m.zakharyaschev@bbk.ac.uk}{https://orcid.org/0000-0002-2210-5183}{}
\authorrunning{A. Artale, A. Gnatenko, V. Ryzhikov, M. Zakharyaschev} 

\Copyright{Alessandro Artale, Anton Gnatenko, Vladislav Ryzhikov, Michael Zakharyaschev} 

 %https://dl.acm.org/ccs/ccs_flat.cfm


\relatedversion{{This is
		an extended version of a paper accepted at ICDT'2025}} %optional, e.g. full version hosted on arXiv, HAL, or other respository/website
\EventEditors{John Q. Open and Joan R. Access}
\EventNoEds{2}
\EventLongTitle{42nd Conference on Very Important Topics (CVIT 2016)}
\EventShortTitle{CVIT 2016}
\EventAcronym{CVIT}
\EventYear{2016}
\EventDate{December 24--27, 2016}
\EventLocation{Little Whinging, United Kingdom}
\EventLogo{}
\SeriesVolume{42}
\ArticleNo{23}
\nolinenumbers
\begin{document}

\maketitle

% !TeX spellcheck = en_GB
\begin{abstract}
  Our concern is the data complexity of answering linear monadic datalog queries whose atoms in the rule bodies can be prefixed by operators of linear temporal logic \LTL. We first observe that, for data complexity, answering any connected query with operators \Next/\Prev (at the \mbox{next/previous} moment) is either in \ACz, or in $\ACCz\!\setminus\!\ACz$, or \NCo-complete, or \LogSpace-hard and in \NLogSpace. Then we show that the problem of deciding \LogSpace-hardness of answering such queries is \PSpace-complete, while checking membership in the classes \ACz and \ACCz as well as \NCo-completeness can be done  in \ExpSpace. Finally, we prove that membership in \ACz or in \ACCz, \NCo-completeness, and \LogSpace-hardness are undecidable 
  % deciding membership in any of \ACz, \ACCz, \NCo-, \LogSpace is undecidable 
  for queries with operators \Df/\Dp (sometime in the future/past) provided that $\NCo \ne \NLogSpace$ and $\LogSpace \ne \NLogSpace$. 
\end{abstract}

% !TEX root =  ../main.tex

\section{Introduction}\label{sec:introduction}

	We consider monadic datalog queries, in which atoms in the rule bodies can be prefixed by the temporal operators $\Next$/$\Prev$ (at the next/previous moment) and $\Df$/$\Dp$ (sometime in the future/past) of linear temporal logic \LTL~\cite{DBLP:books/cu/Demri2016}. This query language, denoted $\mdlnd$, is intended for querying temporal graph databases and knowledge graphs in scenarios such as virus transmission~\cite{Application-Temporal-Virus-Transmission, Application-Temporal-Covid-Contact-Graph}, transport networks \cite{Application-Temporal-Transport-Networks}, social media~\cite{Application-Temporal-Social-Media}, supply chains~\cite{Application-Temporal-Supply-Chains}, and power grids~\cite{Application-Temporal-Power-Grids}.
	In this setting, data instances are finite sets of ground atoms that are timestamped by the moments of time they happen at. The rules in $\mdlnd$ queries are assumed to hold at all times, with time being implicit in the rules and only accessible via temporal operators. 
	We choose \LTL{} for our formalism rather than, say, more expressive metric temporal logic \textsl{MTL}~\cite{DBLP:journals/rts/Koymans90,DBLP:journals/iandc/AlurH93} because \LTL{} has been a well established query language in temporal databases since the 1990s (see~\cite{TSQL2:94,Toman-PODS96,DBLP:journals/tods/ChomickiTB01,Encyclopedia:2018} and the discussion therein on point versus interval-based query languages), also suitable in the context of temporal knowledge graphs as recently argued in \cite{KG-LTL}.
          % and because the timeline in data instances is discrete.

	\begin{example}\label{ex:theseus} 
		Imagine a PhD student working on a paper while hostel hopping. However, finishing the paper requires staying at the same hostel for at least two consecutive nights. Bus services between hostels, which vary from one day to the next, and hostel vacancies are given by a temporal data instance with atoms of the form $\textit{busService}(a, b, n)$ and $\textit{Vacant}(a, n)$, where $a$, $b$ are hostels and $n \in \mathbb N$ a timestamp (see Fig.~\ref{fig:hostels} for an illustration).
		The following $\mdlnd$ query $(\pi_1, \textit{Success})$ finds pairs $(x, t)$ such that having started hopping at hostel $x$ on day $t$, the student will eventually submit the paper: 	
		\begin{align*}
			\pi_1\colon \quad
			\begin{split}
				&\textit{Success}(X) \impd \textit{Vacant}(X) \land \Next \textit{busService}(X, Y) \wedge \Next \textit{Success}(Y),\\
				&\textit{Success}(X) \impd \textit{Vacant}(X) \land \Next \textit{Vacant}(X).
			\end{split}
		\end{align*}
		It is readily seen that answering this query is \NLogSpace-complete for data complexity. If, however, we drop the next-time operator $\Next$ from $\pi_1$, it will become equivalent to $\textit{Vacant}(X)$.
		The next query~$(\pi_2, \textit{Promising})$ simply looks for pairs $(x, t)$ with $x$ having vacancies for two consecutive nights some time later than $t$:
		\begin{align*}
			\pi_2\colon \quad
			\begin{split}
				&\textit{Promising}(X) \impd \Df \textit{Vacant42Nights}(X),\\
				&\textit{Vacant42Nights}(X) \impd \textit{Vacant}(X) \land \Next \textit{Vacant}(X).
			\end{split}
		\end{align*}
		This $\mdlnd$ query can be equivalently expressed as the two-sorted first-order formula
		\begin{equation*}
			\exists t'\, \big( (t < t') \wedge \textit{Vacant}(x, t') \wedge \textit{Vacant}(x, t' + 1) \big),
		\end{equation*}
		where $x$ ranges over objects (hostels) and $t, t'$ over time points ordered by $<$. Formulas in such a two-sorted first-order logic, denoted \FOLess, can be evaluated over finite data instances in \ACz{} for data complexity~\cite{Immerman99}. \hfill $\dashv$
		
		\begin{figure}
			\centering
			\begin{subfigure}{.5\textwidth}
				\centering
				\includegraphics[height=4cm, keepaspectratio]{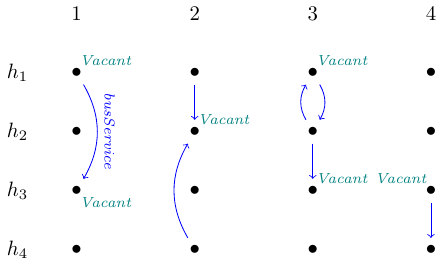}
				\caption{Hostels $h_1, h_2, h_3, h_4$ on days $1, 2, 3, 4$.}
				\label{fig:hostels:data}
			\end{subfigure}%
			\begin{subfigure}{.5\textwidth}
				\centering
				\includegraphics[height=4cm, keepaspectratio]{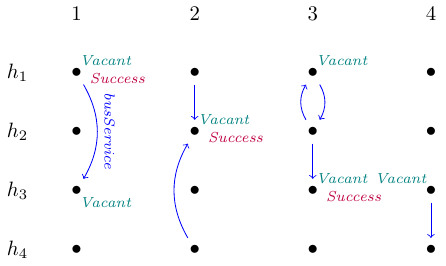}
				\caption{Predicate \textit{Success} inferred by the program $\pi_1$.}
				\label{fig:hostels:inference}
			\end{subfigure}
			\caption{Illustrations for the query $(\pi_1, \textit{Success})$.}
			\label{fig:hostels}
		\end{figure}
	\end{example}
	
	Our main concern is the classical problem of deciding whether a given temporal monadic datalog query is equivalent to a first-order query (over any data instance). In the standard, atemporal database theory, this problem, known as predicate boundedness, has been investigated since the mid 1980s with the aim of optimising and parallelising datalog programs~\cite{DBLP:books/el/leeuwen90/Kanellakis90,DBLP:journals/csur/DantsinEGV01}. Thus, predicate boundedness was shown to be undecidable for binary datalog queries~\cite{DBLP:journals/jlp/HillebrandKMV95} and 2\ExpTime-complete for monadic ones (even with a single recursive rule)~\cite{Cosmadakis-Monadic-DLog-Boundedness-Decidable,DBLP:conf/lics/BenediktCCB15,DBLP:conf/pods/KikotKPZ21}.
	
	Datalog boundedness is closely related to the more general rewritability problem in the ontology-based data access paradigm~\cite{DBLP:journals/jar/CalvaneseGLLR07,DBLP:journals/jair/ArtaleCKZ09}, which brought to light wider classes of ontology-mediated queries (OMQs) and ultimately aimed to decide the data complexity of answering any given OMQ and, thereby, the optimal database engine needed to execute that OMQ. 
	Answering OMQs given in propositional linear temporal logic \LTL{} is either in $\ACz$, or in $\ACCz\setminus\ACz$, or $\NCo$-complete for data complexity~\cite{Straubing94}, the classes well known from the circuit complexity theory for regular languages. 
	For each of these three cases, deciding whether a given \LTL-query falls into it is $\ExpSpace{}$-complete, even if we restrict the language to temporal Horn formulas and atomic queries~\cite{LTL-Rewritability-Checking}. 
	The data complexity of answering atemporal monadic datalog queries comes from four complexity classes $\ACz \subsetneqq \LogSpace \subseteq \NLogSpace \subseteq \PTime$~\cite{DBLP:journals/csur/DantsinEGV01,Cosmadakis-Monadic-DLog-Boundedness-Decidable}. 
	
	Our 2D query language \mdlnd is a combination of datalog and \LTL. It can be seen as the monadic fragment, without negation and aggregate functions, of $\textsc{Dedalus}_0$, a language for reasoning about distributed systems that evolve in time \cite{Dedalus}. It is also close in spirit to temporal deductive databases, TDDs \cite{TDD-1988, TDD-1990}, extending their monadic fragment with the eventuality operator \Df.
	%
	%Although, in general, the addition of a temporal dimension to a logical formalism is notorious for making reasoning harder~\cite{Gabbayetal03}, $\mdlnd$ is designed in such a way that the data complexity of answering its queries is still in $\PTime$. 
	%
	The main intriguing question we would like to answer in this paper is whether deciding membership of 2D queries in, say, $\ACz$ can be substantially harder than deciding membership in $\ACz$ of the corresponding 1D monadic datalog and \LTL{} queries.

	With this in mind, we focus on the sublanguage \lmdlnd of \mdlnd that consists of
		linear queries, that is, those that have at most one IDB (intensional, recursively definable) predicate in each rule body. 
		While the full language inherits from TDDs a \PSpace-complete query answering problem (for data complexity), we prove that linear queries can be answered in \NLogSpace. By $\lmdln$ and $\lmdld$ we denote the fragments of \lmdlnd that only admit the \Next/\Prev and \Df/\Dp operators, respectively. All of our queries are assumed to be connected in the sense that the graph induced by the body of each rule is connected. These fragments retain practical interest: as argued in~\cite{Cosmadakis-Monadic-DLog-Boundedness-Decidable}, atemporal datalog programs used in practice tend to be linear and connected. For example, SQL:1999 explicitly supports linear recursion~\cite{SQL-ISO9075-2-1999}, which together with connectedness is a common constraint in the context of querying graph databases and knowledge graphs \cite{Reutter-Recursion-in-SPARQL, Urzua2019LinearRecursion}, where the focus is on path queries \cite{Cypher, SPARQL11}. 
	
	It is known that deciding whether a linear monadic datalog query can be answered in $\ACz$ is \PSpace-complete~\cite{Cosmadakis-Monadic-DLog-Boundedness-Decidable,DBLP:journals/ijfcs/Meyden00} (without the monadicity restriction, the problem is undecidable \cite{DBLP:journals/jlp/HillebrandKMV95}). The same problem for the propositional \LTL{} fragments of $\lmdln$ and $\lmdlnd$ is also \PSpace-complete~\cite{LTL-Rewritability-Checking}.

	Our main results in this paper are as follows:
	\begin{itemize}
		\item Answering \lmdlnd queries is \NLogSpace-complete for data complexity.
		
		\item It is undecidable whether a given \lmdld-query can be answered in \ACz, \ACCz, or \NCo (if $\NCo \neq \NLogSpace$); it is undecidable whether such a query is \LogSpace-hard (if $\LogSpace \neq \NLogSpace$).
		
		\item Answering any connected $\lmdln$-query is either in $\ACz$ or in $\ACCz \setminus \ACz$, or $\NCo$-complete, or $\LogSpace$-hard---and anyway it is in $\NLogSpace$.
		\item It is \PSpace-complete to decide whether a connected $\lmdln$-query is \LogSpace-hard; checking whether it is in \ACz, or in $\ACCz\setminus\ACz$, or is \NCo-complete can be done in \ExpSpace.
	\end{itemize}
	(Note that dropping the past-time operators \Prev and \Dp from the languages has no impact on these complexity results.)
	Thus, the temporal operators \Next/\Prev and \Df/\Dp exhibit drastically different types of interaction between the object and temporal dimensions. To illustrate the reason for this phenomenon, consider first the following $\lmdln$-program: 
	\begin{align}
		&G(X) \impd A(X) \wedge \Next R(X, Y) \wedge \Next D(Y), \label{rule:vertical}\\
		&D(X) \impd \Next D(X), \label{rule:horizontal:recursive}\\
		&D(X) \impd B(X). \label{rule:horizontal:fin}
	\end{align}
	Suppose a data instance consists of timestamped atoms $A(a, 0)$, $R(a, b, 1)$, $B(b, 5)$. We obtain $G(a, 0)$ by first applying rule~\eqref{rule:horizontal:fin} to infer $D(b, 5)$, then  rule~\eqref{rule:horizontal:recursive} to infer $D(b, 4), D(b, 3), D(b, 2)$, and $D(b, 1)$, and finally rule~\eqref{rule:vertical} to obtain $G(a, 0)$. Rules~\eqref{rule:horizontal:recursive} and~\eqref{rule:horizontal:fin} are applied along the timeline of a single object, $b$, while the final application passes from one object, $b$, to another, $a$. To do so, we check whether a certain condition holds for the joint timeline of $a$ and $b$, namely, that they are connected by $R$ at time 1. However, if we are limited to the operators \Next/\Prev, the number of steps that we can investigate along such a joint timeline is bounded by the maximum number of nested \Next/\Prev in the program. Therefore, there is little interaction between the two phases of inference that explore the object and the temporal domains. In contrast, rules with \Df/\Dp can inspect both dimensions simultaneously as, for example, the rule  
	\begin{align}\label{rule:both-dimensions}
		&G(X) \impd \Df R(X, Y) \wedge D(Y).
	\end{align}
	In this case, inferring $G(a, 0)$ requires checking the existence of an object $b$ satisfying $D(b, 0)$ and $R(a, b, \ell)$ at some \emph{arbitrarily distant} moment $\ell$ in the future. Given that our programs are monadic, the predicate $\Df R(X, Y)$ 
	cannot be expressed using operators \Next/\Prev only.
	% Since our programs are monadic, this cannot be expressed using operators \Next/\Prev only.
	
	Our positive results are proved by generalising the automata-theoretic approach of~\cite{Cosmadakis-Monadic-DLog-Boundedness-Decidable}. As a by-product, we obtain a method to decompose every connected $\lmdln$-query $(\pi, G)$ into a plain datalog part~$(\pi_d, G)$ and a plain \LTL part~$(\pi_t, G)$, which are, however, 
	substantially larger than $(\pi, G)$, so that the data complexity of answering $(\pi, G)$ equals the maximum of the respective data complexities of $(\pi_d, G)$ and $(\pi_t, G)$. This reinforces the `weakness of interaction' between the relational and the temporal parts of the query when the latter is limited to operators \Next/\Prev. We also provide some evidence that, in contrast to the atemporal case, the automata-theoretic approach cannot be generalised to the case of disconnected queries.
	The undecidability of the decision problem for \lmdld-queries is proved by a reduction of the halting problem for Minsky machines with two counters~\cite{Minsky}.
	
	%%%%%%%%%%%%%%%%%%%%%%%%%%%%%%%%%%%%%%%%
	
	The paper is organised as follows.
	In Section~\ref{sec:preliminaries}, we give formal definitions of data instances and queries, and prove that every temporal monadic datalog query can be answered in \PTime, and in \NLogSpace if it is linear.
	In Section~\ref{sec:diamond-undecidability}, we show that checking whether a given query with operator \Df or \Dp has data complexity lower than \NLogSpace is undecidable.
	Section~\ref{sec:automata} considers $\lmdln$-queries by presenting a generalisation of the automata-theoretic approach of~\cite{Cosmadakis-Monadic-DLog-Boundedness-Decidable}, which is then used
	in Section~\ref{sec:next} to provide the decidability results. We conclude with a discussion of future work and our final remarks in Section~\ref{sec:conclusions}.	

        %%%%%%%%%%%%%%%%%%%%%%%%%%%%%%%%%%%%%%%%%%%%%%%%%%%%%%%%%%%%%%%%%%%%%%

% !TEX root =  ../main.tex

\section{Preliminaries}\label{sec:preliminaries}

A \textit{relational schema} $\Sigma$ is a finite set of \textit{relation symbols} $R$ with associated arities $m \geq 0$. A \textit{database} $D$ over a schema $\Sigma$ is a set of ground atoms $R(d_1, \dots, d_m)$, $R \in \Sigma$, $m$ is the arity of $R$. We call $d_i$, $1 \leq i \leq m$, the \textit{domain objects} or simply objects. We denote by $\Delta_D$ the set of objects occurring in $D$. We denote by $|D|$ the number of atoms in $D$.
We denote by $[a, b]$ the set of integers $\{m \mid a \leqslant m \leqslant b\}$, where $a, b \in \Z$.
A \textit{temporal database} $\Db$ over a schema $\Sigma$ is a finite sequence $\langle D_{l}, D_{l+1} \dots, D_{r-1}, D_{r}\rangle$ of databases over this schema for some $l < r$, $l, r \in \Z$.
Each database $D_i, l \leqslant i \leqslant r,$ is called the $i$'th \textit{slice} of $\Db$ and $i$ is called a \textit{timestamp}. We denote $[l, r]$ by $\tem(\Db)$.
The \textit{size} of $\Db$, denoted by $|\Db|$, is the maximum between $|\tem(\Db)|$ and $\max \{|D_l|, \dots, |D_r| \}$. The domain of the temporal database $\Db$ is $\bigcup_{l \leq i \leq r} \Delta_{D_i}$ and is denoted by $\Delta_{\Db}$.
A homomorphism from $\Db$ as above to $\Db' = \langle D_{l'}, D_{l'+1} \dots, D_{r'-1}, D_{r'}\rangle$ is a function $h$ that maps $\Delta_\Db \cup [l,r]$ to $\Delta_{\Db'} \cup [l',r']$ so that $R(d_1, \dots, d_m) \in D_\ell$ if and only if $R(h(d_1), \dots, h(d_m)) \in D'_{h(\ell)}$.

We will deal with \textit{temporal conjunctive queries} (temporal CQs) that are formulas of the form $Q(\Xol) = \exists \overline{U} \p(\overline{X}, \overline{U})$, where $\Xol, \Uol$ are tuples of variables and $\p$, the \textit{body} of the query, is defined by the following BNF:
\begin{equation}\label{bnf:preliminaries:tcq}
		\p \Coloneqq
		R(Z_1, \dots, Z_m)
		\mid
		% (\p \vee \p)
		% \mid
		(\p \wedge \p)
		\mid
		\Omc \p
\end{equation}
where $R \in \Sigma$ and $Z_1, \dots, Z_m$ are variables of $\Xol \cup \Uol$, and $\Omc$ is any of $\Next, \Prev, \Df$ and $\Dp$ (the operators \Next/\Prev mean `at the next/previous moment' and \Df/\Dp `sometime in the future/past'). For brevity, we will use the notation $\Omc^n$, $\Omc \in \{\Next, \D\}$, for a sequence of $n$ symbols $\Omc$ if $n > 0$, of $|n|$ symbols $\Omc^{-}\in \{\Prev, \Dp\}$ if $n < 0$, and an empty sequence if $n = 0$. We call \Xol the \textit{answer variables} of $Q$.
To provide the semantics, we define $\Db, \ell \mdl \varphi(d_1, \dots, d_m)$ for $\ell \in \Z$ and $d_1, \dots, d_m \in \Delta_\Db$ as follows:
\begin{align}
    &\Db, \ell \mdl  R(d_1, \dots, d_m) \iff \ell \in [l,r] \text{ and } R(d_1, \dots, d_m) \in D_\ell\\
	&\Db, \ell \mdl \p_1 \wedge \p_2 \iff \Db, \ell \mdl \p_1 \text{ and } \Db, \ell \mdl \p_2\label{eq:preliminaries:models-relation:conjunction} \\
	&\Db, \ell \mdl \Next \p \iff \Db, {\ell + 1} \mdl \p \label{eq:preliminaries:models-relation:next} \\
	&\Db, \ell \mdl \Df \p \iff \Db, {\ell'} \mdl \p \text{ for some } \ell' > \ell \label{eq:preliminaries:models-relation:diamond}
\end{align}
and symmetrically for $\Prev$ and $\Dp$.
Given a temporal database $\Db$, a timestamp $\ell \in \Z$, and a query $Q(\Xol) = \exists \overline{U} \p(\overline{X}, \overline{U})$, we say that $\Db, \ell \mdl Q(d_1, \dots, d_k)$ if there exist $\delta_1, \dots, \delta_s \in \Delta_\Db$ such that $\Db, \ell \mdl \p(d_1, \dots, d_k, \delta_1, \dots, \delta_s)$, where $k = |\Xol|, s = |\Uol|$.

The problem of \textit{answering} a temporal CQ $Q$ is to check, given $\Db$, $\ell \in \tem(\Db)$, and $\bar d = \langle d_1, \dots, d_k\rangle$, whether $\Db, \ell \mdl Q(\bar d)$. Answering temporal CQs is not harder than that for non-temporal CQs. Indeed, we show that any $Q$ is \FOLess-rewritable in the sense that there exists an \FOLess-formula $\psi(\Xol, t)$ such that for all $\ell$ and $\bar d$ as above $\Db, \ell \mdl Q(\bar d)$  whenever $\psi(\bar d, \ell)$ is true in the two-sorted first-order structure $\mathfrak S_\Db$, whose domain is $\Delta_\Db \cup \tem(\Db)$, and where $R(\bar d, \ell)$ is true whenever $\Db, \ell \mdl R(\bar d)$ and $(\ell < \ell')$ is true whenever $\ell < \ell'$ (see~\cite{LTL-Rewritability} for details). It follows by~\cite{Immerman99} that the problem of answering a temporal CQ $Q$ is in \ACz.
	
	We outline the construction of the rewriting. Let $Q(\Xol) = \exists \overline{U} \p(\overline{X}, \overline{U})$.  The main issue with the construction is that temporal subformulas of $\varphi$, say $\D \varkappa$, may be true for $\bar d$ at $\ell \in \tem(\Db)$, because $\Db, \ell' \models \varkappa(\bar d)$ for $\ell' \not \in \tem(\Db)$. Had that not been the case, we could construct the rewriting for $Q$ straightforwardly by induction of $\varphi$. To overcome this, let $N$ be the number of temporal operators in $\varphi$. We use a property that for all tuples $\bar d$ of objects from $\Delta_\Db$ and subformulas $\varkappa$ of $\varphi$, we have
	\begin{equation}\label{eq:loop}
		\begin{split}
			&\Db, r + N + 1 \mdl \varkappa(\bar d) \iff \Db, \ell \mdl \varkappa(\bar d) \text{ for all } \ell > r + N\\
			&\Db, l - N - 1 \mdl \varkappa(\bar d) \iff \Db, \ell \mdl \varkappa(\bar d) \text{ for all } \ell < l - N
		\end{split}
	\end{equation}
	%
	%Thus, it is enough for the rewriting to look at most $N + 1$ steps away from the borders.
	%
	Thus, for any subformula $\varkappa(\Zol)$ of $\varphi$, we construct, by induction, the formulas $\psi_\varkappa(\Zol, t)$  and $\psi_\varkappa^{i}(\Zol)$ for $i \in [-N - 1, \dots, -1] \cup [1, \dots, N + 1]$,
	%$i \in [-(N -N_\varkappa), \dots, -1] \cup [1, \dots, N-N_\varkappa]$,
	%where $N_\varkappa$ is the number of temporal operators in $\varkappa$,
	so that for any \Db, $\tem(\Db) = [l, r]$, and any objects $\bar d \in \Delta_\Db^{|\Zol|}$,
\begin{align*}
&\Db, \ell \models \varkappa(\bar d) &&\iff &&\Smf_\Db \models \psi_\varkappa(\bar d, \ell) &&\text{ for }\ell \in [l,r],\\
&\Db, (r + i) \models \varkappa(\bar d) &&\iff &&\Smf_\Db \models \psi_\varkappa^i(\bar d) &&\text{ for }1 \leqslant i \leqslant N + 1,\\
&\Db, (l + i) \models \varkappa(\bar d) &&\iff &&\Smf_\Db \models \psi_\varkappa^i(\bar d) &&\text{ for }-N - 1 \leqslant i \leqslant -1.
\end{align*}
For the base case, we set $\psi_R(\Zol, t) = R(\Zol, t)$ and $\psi_R^i(\Zol) = \bot$. For an induction step, e.g., $\psi_{\Next \varkappa}^{-1}(\Zol) = \psi_{\varkappa}(\Zol, \min)$,
	$\psi_{\Next \varkappa}^{N + 1}(\Zol) = \psi^{N + 1}_{\varkappa}(\Zol)$, $\psi_{\Next \varkappa}^i(\Zol) = \psi_{\varkappa}^{i+1}(\Zol)$ for all other $i$, and finally  $\psi_{\Next \varkappa}(\Zol, t) = \exists t' ((t' = t + 1) \land \psi_{\varkappa}(\Zol, t')) \lor ((t = \max) \land \psi_{\varkappa}^1(\Zol))$.
	Here, $\min$ and $\max$ are defined in $\FOLess$ as, respectively, $<$-minimal and $<$-maximal elements in $\tem(\Db)$ ($(t' = t + 1)$ is $\FOLess$-definable as well).
	%Further details are given in the appendix.
	The required rewriting of $Q$ is then the formula $\exists \Uol \psi_\varphi(\Xol, \Uol, t)$.
\begin{restatable}{proposition}{PropACz}\label{prop:preliminaries:tcq-complexity}
	For a temporal CQ $Q(X_1, \dots, X_k)$, checking $\Db, \ell \mdl Q(d_1, \dots, d_k)$ is in \ACz for data complexity.
\end{restatable}

In the non-temporal setting, a body of a CQ (a conjunction of atoms), can be seen as a database (a set of atoms). We have a similar correspondence in the temporal setting, for queries without operators \Df and \Dp.
Indeed, observe that $\Next^a(\p_1 \wedge \p_2) \equiv \Next^a \p_1 \wedge \Next^a \p_2$ and $\Next \Prev \p \equiv \Prev\Next\p \equiv \p$. Hence we can assume that any temporal CQ body is  a conjunction of temporalised atoms of the form $\Next^k R(Z_1, \dots, Z_m)$. Given a temporal CQ $Q$ of this form, let $l$ be the least and $r$ the greatest number such that $\Next^{l}R(\Zol)$ and $\Next^r R'(\Zol')$ appear in $Q$, for some $R, R', \Zol$ and $\Zol'$.
Let $\Db_Q$ be a temporal database whose objects are the variables of $Q$, and $\tem(\Db)$ equals $[l, r]$ if $0 \in [l, r]$, $[0, r]$ if $0 < l$, and $[l, 0]$ if $r < 0$. Furthermore, let $R(\Zol) \in D_\ell$ whenever $\Next^\ell R(\Zol)$ is in $Q$, $\ell \in \tem(\Db)$. Then we can, just as in the non-temporal case, characterise the relation $\mdl$ in terms of homomorphisms:
\begin{restatable}{lemma}{LemmaHomo}\label{lm:preliminaries:tcq-next-homomorphisms}
	For any temporal CQ $Q(X_1, \dots, X_k)$ without \Df and \Dp, $\Db, \ell \mdl Q(d_1, \dots, d_k)$ if and only if there is a homomorphism $h$ from $\Db_Q$ to $\Db$ such that $h(X_i) = d_i, 1 \leqslant i \leqslant k$, and $h(0) = \ell$.
\end{restatable}

%%%%%%%%%%%%%%%%%%%%%%%%%%%%%%%%%%%%%%%%%%%%%%%%%%%%%%%%%%%%
\subsection{Temporal Datalog}\label{sec:preliminaries:mtdlog}

We define a temporalised version of datalog to be able to use recursion in querying temporal databases. We call this language \textit{temporal datalog}, or \dlnd. A \textit{rule} of \dlnd over a schema $\Sigma$ has the form:
\begin{equation}\label{def:tdlog-rule}
	C(\ovl{X}) \impd \Omc^* R_1(\ovl{U}_1) \wedge \dots \wedge \Omc^* R_s(\ovl{U}_s)
\end{equation}
where $R_i$ and $C$ are relation symbols over $\Sigma$ and $\Omc^*$ is an arbitrary sequence of temporal operators $\Next, \Prev, \Df$ and $\Dp$. The part of the rule to the left of the arrow is called its \textit{head} and the right-hand side---its \textit{body}. All variables from the head must appear in the body.

A \textit{program} is a finite set of rules. The relations that appear in rule heads constitute its IDB schema, \IDBpi, while the rest form the EDB schema, \EDBpi. A rule is linear if its body contains at most one IDB atom and \textit{monadic} if the arity of its head is 1.
A program $\pi$ is linear (monadic) if so are all its rules. We say that the program is in plain datalog if it does not use the temporal operators and in plain \LTL if all its relations have arity 1 and every rule uses just one variable.
\textit{Recursive rules} are those that contain IDB atoms in their bodies, other rules are called \textit{initialisation rules}. The \textit{arity} of a program is the maximal arity of its IDB atoms.

Our results are all about \textit{connected} programs. Namely, define the \textit{Gaifman graph} of a temporal CQ to be a graph whose nodes are the variables and where two variables are connected by an edge if they appear in the same atom. A rule body is connected if so is its Gaifman graph, and a program is connected when all rules are connected.
The \textit{size} of a program $\pi$, denoted $|\pi|$, is the number of symbols needed to write it down, where every relation symbol $R \in \EDBpi \cup \IDBpi$ is counted as one symbol, and a sequence of operators the form $\Omc^k$ is counted as $|k|$ symbols.

When a program $\pi$ is fixed, we assume that all temporal databases that we work with are defined over \EDBpi. So let $\pi$ be a program and $\Db$ a temporal database. An \textit{enrichment} of $\Db$ is an (infinite) temporal database $\Emc = \langle E_\ell \rangle_{\ell \in \Zbb}$ over the schema $\EDBpi \cup \IDBpi$ such that $\Delta_\E = \Delta_\Db$ and for any $R \in \EDBpi$ and any $\ell \in \Zbb$, $R(d_1, \dots, d_m) \in E_\ell$ if and only if $R(d_1, \dots, d_m) \in D_\ell$. Thus, the only EDB atoms in \E are those of \Db, but \E `enriches' \Db with various IDB atoms at different points of time. We say that $\E$ is a \textit{model} of $\pi$ and $\Db$ if $(i)$ $\E$ is an enrichment of $\Db$; and $(ii)$ for any rule $C(\ovl{X}) \impd \psi(\Xol, \Uol)$ of $\pi$, $\E \models C(\ovl{X}) \impd \psi(\Xol, \Uol)$, i.e., for all $\ell\in\Zbb$ and any tuples $\bar d \in \Delta_\E^{|\Xol|}, \bar \delta \in \Delta_\E^{|\Uol|}$,  $\E,\ell\models \psi(\bar d, \bar \delta)$ implies $\E,\ell\models C(\bar d)$. We write $\Db,\pi,\ell\models C(\bar d)$ if for every model $\E$ of $\pi$ and \Db it follows that $\E, \ell\models C(\bar d)$.

A \dlnd \textit{query} is a pair $(\pi, G)$, where $\pi$ is a \dlnd program and $G$ an IDB atom, called the \textit{goal predicate}. The \textit{arity} of a query is the arity of $G$.
Given a temporal database $\Db$, a timestamp $\ell\in\tem(\Db)$, a tuple $(d_1,\ldots,d_k)\in\Delta_\Db^k$ and a \dlnd query $(\pi, G)$ of arity $k$, the pair $\langle (d_1,\ldots,d_k), \ell \rangle$ is a \emph{certain answer} to $(\pi, G)$ over $\Db$ if $\Db,\pi,\ell\models G(d_1,\ldots,d_k)$.
The \emph{answering problem} for a \dlnd query $(\pi, G)$ over a temporal database $\Db$ is that of checking, given a tuple $(d_1, \dots, d_k)\in\Delta_\Db^k$, and $\ell\in \tem(\Db)$, if $\langle (d_1,\ldots,d_k), \ell \rangle$ is a certain answer to $(\pi, G)$ over $\Db$. We use the term `complexity of the query $(\pi, G)$' to refer to the data complexity of the associated answering problem, and say, e.g., that $(\pi, G)$ is complete for polynomial time (for \PTime) or for nondeterministic logarithmic space (\NLogSpace) if the answering problem for $(\pi, G)$ is such.

Our main concern is how the data complexity of the query answering problem is affected by the features of $\pi$.
% affect it.
The following theorem relies on similar results obtained for temporal deductive databases~\cite{TDD-1988, TDD-1990} and temporal description logics~\cite{AKRZ:LPAR13,Cookbook,Gutierrez-Basulto16}.
\begin{restatable}{theorem}{PropinNL}\label{prop:preliminaries:mtdlog-complexity}
Answering \textup{(}monadic\textup{)} \dlnd queries is \textup{\PSpace}-complete for data complexity; answering linear \textup{(}monadic\textup{)} \dlnd queries is \textup{\NLogSpace}-complete.
\end{restatable}
\begin{proof}
  (\textit{Sketch}) For full \dlnd, \PSpace-completeness
 % of this problem, for data complexity,
  can be shown by reusing the techniques for temporal deductive databases \cite{TDD-1988, TDD-1990}. However, we prove that a \textit{linear} query can be answered in \NLogSpace. Indeed, fix a linear query $(\pi, G)$. Without loss of generality, we assume that temporalised IDB atoms in rule bodies of $\pi$ have the form $\Next^k C(\Yol)$, where $|k| \leqslant 1$ (the cases of \Df/\Dp and consecutive \Next/\Prev can be expressed via recursion). Given a temporal database \Db, tuples of objects $\avec{c}$ and $\avec{d}$ from $\Delta_{\Db}$, and $\ell \in \Zbb$, we write $C(\avec{c}) \leftarrow_{\ell,k} D(\avec{d})$ if $\pi$ has a rule $C(\Xol) \impd \p(\Xol, \Yol, \Uol) \wedge \Next^k D(\Yol)$ such that $\Db, \ell \mdl \exists \Uol \p(\avec{c}, \avec{d}, \Uol)$. Analogously, we write $C(\avec{c}) \leftarrow_\ell$ if there is an initialisation rule $C(\Xol) \impd \p(\Xol, \Uol)$ and $\Db, \ell \mdl \exists \Uol \p(\avec{c}, \Uol)$. Then, given a linear \dlnd query  $(\pi, G)$, we have $\Db, \pi, \ell \mdl G(\avec{d})$ if and only if $G = C_0$, $\avec{d} = \avec{c}_0$, and there exists a sequence
\begin{align}\label{eq:derivation}
	C_0(\avec{c}_0) \leftarrow_{k_0} C_1(\avec{c}_1) \leftarrow_{k_1} \dots \leftarrow_{k_{n-1}} C_{n}(\avec{c}_{n})
\end{align}
such that $C_i \in \textit{IDB}(\pi)$, tuples $\avec{c}_i$ are from $\Delta_{\Db}$, $\ell_0 = \ell$, $\ell_{i+1} = \ell_{i}+k_i$ for $0 \leq i < n$, $C_i(\avec{c}_i) \leftarrow_{\ell_i, k_i} C_{i+1}(\avec{c}_{i+1})$, and $C_n(\avec{c}_n) \leftarrow_{\ell_n}$. Let $\tem(\Db) = [l, r]$. Using property \eqref{eq:loop}, we observe that a rule $C(\avec{c}) \leftarrow_{\ell,k} D(\avec{d})$ either holds or does not hold simultaneously for all $\ell > r + N$, where $N$ is the number of temporal operators in $\pi$, and, similarly, for all $\ell < l - N$. Now, consider a sequence~\eqref{eq:derivation} and $\ell$, where all $\ell_i > r + N$
% $\ell_i > \max \tem(\Db) + N$
(the case for $\ell < l - N$
% $\ell < \min \tem(\Db) - N$
is analogous). Any loop of the form $C_i(\avec{c}_i) \leftarrow_{k_i} \dots \leftarrow_{k_{j-1}} C_j(\avec{c}_j)$ in it with $C_i(\avec{c}_i) = C_j(\avec{c}_j)$ %such that $\sum_{t=0}^{i-1} k_t < \sum_{t=0}^{j-1} k_t$ can be repeated. Moreover, any such a loop 
can be removed as long as in the resulting sequence
\begin{equation*}
C_0(\avec{c}_0) \leftarrow_{k_0} C_1(\avec{c}_1) \leftarrow_{k_1} \dots \leftarrow_{k_{i-1}} C_i(\avec{c}_i) \leftarrow_{k_j} C_{j+1}(\avec{c}_{j+1}) \dots \leftarrow_{k_{n-1}} C_{n}(\avec{c}_{n}),
\end{equation*}
the sum of all $k_t$ remains $\geq 0$. This allows us to convert any sequence \eqref{eq:derivation} to a sequence with the same $C_0$ in the beginning, the same $C_n$ in the end, and where all $\ell_i$ do not exceed $\ell + O(|\IDBpi|\cdot|\Delta_\Db|^a)$, for $a$ equal to the maximal arity of a relation in $\IDBpi$.
This means that, for $\ell \in \tem(\Db)$, we can check $\Db, \pi, \ell \mdl G(\avec{d})$ using timestamps in the range $[l - O(|\pi|\cdot|\Delta_\Db|^a), r + O(|\pi|\cdot|\Delta_\Db|^a)]$.
Clearly, the existence of such a derivation can be checked in \NLogSpace{}.
\end{proof}

However, individual queries may be easier to answer than in the general case. Since \dlnd combines features of plain datalog and of linear temporal logic, its queries can correspond to a variety of complexity classes. Recall, for example, queries $(\pi_1, \textit{Good})$ and $(\pi_2, \textit{Satisfactory})$ from Section \ref{sec:introduction}, the first of which is hard for logarithmic space (\LogSpace-hard) and the second lies in \ACz, the class of problems decidable by unbounded fan-in, polynomial size and constant depth boolean circuits. Furthermore, by using unary relations and operator~\Next, one can simulate
%PARITY problem \cite{Parity}, and, in general,
any regular language, giving rise to queries that lie in $\ACCz$, the class obtained from \ACz by allowing `MOD $m$' gates, or are complete for $\NCo$, the class defined similarly for bounded fan-in polynomial circuits of logarithmic depth. Intuitively, the problems in \ACz, \ACCz, and \NCo are solvable in short (constant or logarithmic) time on a parallel architecture; see \cite{Straubing94} for more details.

In the remainder of the paper we focus on deciding the data complexity for the linear monadic fragment of \dlnd, denoted \lmdlnd.
It is well-known that a plain datalog query can be characterised via an infinite set of conjunctive queries called its expansions~\cite{Naughton-Expansions}. We define expansions for \lmdlnd and use them as the main tool in our (un)decidability proofs.

%%%%%%%%%%%%%%%%%%%%%%%%%%%%%%%%
\subsubsection{Expansions for Linear Monadic Queries.} \label{sec:preliminaries:mtdlog:expansions}

Let $\pi$ be a \lmdlnd program and $Q(X)$ be a unary temporal conjunctive query with a single answer variable and containing a unique IDB atom, say $D(Y)$, from $\pi$. Let $P(X)$ be another temporal conjunctive query with a single answer variable,
and let $P'(Y)$ be obtained from $P(X)$ by substituting $X$ by $Y$ and all other variables with fresh ones.
A \textit{composition} of $Q$ and $P$, denoted $Q \circ P$, has the form of $Q$ with $D(Y)$ substituted with $P'(Y)$. We note that the variables of $Q$ remain present in $Q \circ P$ and $X$ is an answer variable of $Q \circ P$. If $P$ contains an IDB atom and $K(X)$ is another temporal conjunctive query, the composition can be extended in the same fashion to $(Q \circ P) \circ K$, and so on. Note that, up to renaming of variables, $(Q \circ P) \circ K$ and $Q \circ (P \circ K)$ are the same queries, so we will omit the brackets and write $Q \circ P \circ K$.

\emph{Expansions} are compositions of rule bodies of a program $\pi$. Let $B_1, B_2, \dots, B_{n - 1}$ be such that $B_i$ is the body of the recursive rule:
\begin{equation}\label{rule:for-expansions-rec}
	C_i(X) \impd A_i(X, Y, U_1, \dots, U_{m_i})\ \wedge\ \Next^{k_i} C_{i + 1}(Y),
\end{equation}
and $B_n$ is the body of an initialization rule
\begin{equation}\label{rule:for-expansions-init}
	C_n(X) \impd B(X, V_1, \dots, V_{m_n}).
\end{equation}
The composition $B_1 \circ \dots \circ B_n$ is called an \textit{expansion} of $(\pi, C_1)$, and $n$ is its \textit{length}. The set of all expansions of $(\pi, C_1)$ is denoted $\expand(\pi, C_1)$. Moreover, let $\Gamma^r_\pi$ be the set of all recursive rule bodies of $\pi$ and $\Gamma^i_\pi$ be the set of all initialization rule bodies. Then each expansion can be regarded (by omitting the symbol $\circ$) as a word in $(\Gamma^r_\pi)^* \Gamma^i_\pi$, and $\expand(\pi, C_1)$ as a sublanguage of $(\Gamma^r_\pi)^* \Gamma^i_\pi$. Adopting a language-theoretic notation, we will use small Latin letters $w, u, v$, etc. to denote expansions. To highlight the fact that each expansion is a temporal conjunctive query with the answer variable $X$, we sometimes write $w(X) \in \expand(\pi, C)$.

It is a direct generalization from the case of plain datalog \cite{Naughton-Expansions} that $\Db, \pi, \ell \mdl C(d)$ if and only if there exists $w(X) \in \expand(\pi, C)$ such that $D_\ell\mdl w(d)$.

%%%%%%%%%%%%%%%%%%%%%%%%%%%%%%%%%%%%%%%%%%%%%%%%%%%%%%%%%%%%%%%%%%%%%%

% !TEX root =  ../main.tex

\section{Undecidability for Queries with \Df}
\label{sec:diamond-undecidability}

Our first result about deciding the complexity of a given query is negative. 
%
%\begin{theorem}\label{thm:undecidability-for-diamond}
%	For connected \lmdld and for any complexity class $\Cmc$ closed under reductions computable in \ACz and such that $\ACz \sbs \Cmc \subsetneqq \NLogSpace$, it is undecidable whether a given query can be answered in $\Cmc$ for data complexity.
%\end{theorem}
%
	\begin{theorem}\label{thm:undecidability-for-diamond}
		It is undecidable whether a given \lmdld-query can be answered in \ACz, \ACCz, or \NCo (if $\NCo \neq \NLogSpace$). It is undecidable whether the query is \LogSpace-hard (if $\LogSpace \neq \NLogSpace$).
	\end{theorem}
The proof is by a reduction from the halting problem of 2-counter machines~\cite{Minsky}. Namely, given a 2-counter machine $M$ we construct a query $(\pi_M, G)$ that is in \ACz if $M$ halts and \NLogSpace-complete otherwise.

Recall, that a
\textit{2-counter machine} is defined by a finite set of states $S = \{s_0, \dots, s_n\}$,
with a distinguished \textit{initial state} $s_0$,
%and \textit{final state} $s_n$,
two counters able to store
non-negative integers, and a transition function~$\Theta$. On each
step the machine performs a \textit{transition} by changing its
state and incrementing or decrementing each counter by 1 with a
restriction that their values remain non-negative. The next
transition is chosen according to the current state and the values
of the counters. However, the machine is only allowed to perform
zero-tests on counters and does not distinguish between two
different positive values.
Formally, transitions are given by a partial function $\Theta$:
\begin{equation}\label{def:machine:theta}
	S \times \{0, +\} \times \{0, +\} \rightarrow S \times
	\{-1, 0, 1\} \times \{-1, 0, 1\}.
\end{equation}
%
%defined for each triple $(s, a, b)$ except for the case $s = s_{\textup{fin}}$.
%
Let $\mathrm{sgn}(0) = 0$ and $\mathrm{sgn}(k) = +$ for all $k > 0$. A
\textit{computation} of $M$ is a sequence of \textit{configurations}:
\begin{equation}\label{def:machine:computation}
	(s_0, a_0, b_0), (s_1, a_1, b_1), (s_2, a_2, b_2) \dots (s_m, a_m, b_m),
\end{equation}
such that for each $i, 0 \leqslant i < m$, holds
$\Theta(s_i, \mathrm{sgn}(a_i), \mathrm{sgn}(b_i)) = (s_{i + 1},
\varepsilon_1, \varepsilon_2)$ and $a_{i + 1} = a_i + \varepsilon_1$,
$b_{i + 1} = b_i + \varepsilon_2$. We assume that
$a_0, b_0 = 0$ and $\Theta$ is such that
$a_i, b_i \geqslant 0$ for all~$i$.
%, and that there is no transition from the final state $s_n$.
We call $m$ the \textit{length} of the computation.
%
%An \textit{initial} computation is
%one with $s_0 = s_{\textup{init}},\ a_0, b_0~=~0$.
%%
We say that $M$ \textit{halts} if $\Theta(s_m, \mathrm{sgn}(a_m), \mathrm{sgn}(b_m))$ is not defined in a computation. Thus, $M$ either halts after $m$ steps or goes through infinitely many configurations.

Let $M$ be a 2-counter machine. We construct a connected linear query $(\pi_M, G)$ using the operator \Df only, such that its evaluation is in \ACz if $M$ halts, but becomes \NLogSpace-complete otherwise. The construction with the operator \Dp instead of \Df is symmetric.

We set the EDB schema $\Sigma = \{ T, U_1, U_2\}$ of three relations which stand for 'transition', 'first counter update', and 'second counter update', respectively. Intuitively, domain elements of a temporal database represent configurations of $M$, and role triples of $T$, $U_1$ and $U_2$, arranged according to certain rules described below, will play the role of transitions. A sequence of nodes connected by such triples will thus represent a computation of $M$. Our program $\pi_M$ will generate an expansion along such a sequence, trying to assign to each configuration an IDB that represents a state of $M$, which will be possible while the placement of the connecting roles on the temporal line follows the rules of $\Theta$. If the machine halts, there is a maximum number of steps we can make, so the check can be done in \ACz. If $M$ does not halt, however, it has arbitrarily long computations, and the query evaluation becomes \NLogSpace-complete.

Here are the details.
A configuration $(s, a, b)$ is represented by an object $d$ and three timestamps $\ell_0, \ell_1, \ell_2$ such that $\ell_1 = \ell_0 + a$ and $\ell_2 = \ell_0 + b$. The values of the counters~$a$ and $b$ are indicated, respectively, by existence of connections $U_1^{\ell_1}(d, d')$ and $U_2^{\ell_2}(d, d')$ to some object $d'$ that is supposed to represent the next configuration in the computation. For the transition to happen, we also require $T^{\ell_0}(d, d')$.
Given a computation of the form~\eqref{def:machine:computation}, the corresponding \textit{computation path} is a pair $(\lp d_0, \dots, d_n \rp, \ell_0)$, where for each~$i, 0 \leqslant i < n$, there are $T^{\ell_0}(d_i, d_{i + 1})$, $U_1^{\ell_0 + a_i}(d_i, d_{i + 1})$, and $U_2^{\ell_0 + b_i}(d_i, d_{i + 1})$ in the database, with an additional requirement that $a_0 = b_0 = 0$. Two types of problems may be encountered on such a path. A \textit{representation violation} occurs when an object has more than one outgoing edge of type $T$, $U_1$, or $U_2$. A \textit{transition violation} of type $(s_i, \alpha, \beta)$ is detected when there are two consecutive nodes $d_i, d_{i + 1}$, $\alpha = \sgn(a_i), \beta = \sgn(b_i)$, and $\Theta(s_i, \alpha, \beta) \neq (s_{i + 1}, a_{i + 1} - a_i, b_{i + 1} - b_i)$.
The program $\pi_M$ will look for such violations. It will have an IDB relation symbol $S_i$ per each state $s_i$ of~$M$. The initialisation rules allow to infer any state IDB from a representation violation, and the IDB $S_i$ from a transition violation of type $(s_i, \alpha, \beta)$. The recursive rules then push the state IDB up a computation path following the rules of $\Theta$ and tracing 'backwards' a computation of $M$. Once $\pi_M$ infers the initial state IDB at a position with zero counter values, we are done. It remains to give the rules explicitly.

We use a shortcut $\Drefl$ to mean a `reflexive' version of \D, i.e. $\Drefl \p \equiv \D \p \vee \p$. Clearly, every rule $\Drefl$ can be rewritten to an equivalent set of rules without it. We need the rules: 
%\nb{V: the axiom right below is not needed}
%
\begin{align}
	&S_i(X) \impd \Drefl RV(X),\quad 0 \leqslant i \leqslant n\ 
	&&RV(X) \impd T(X, Y) \wedge \D T(X, Z) \label{rule:machine:representation-violation-first}\\
	&RV(X) \impd U_1(X, Y) \wedge \D U_1(X, Z), \label{rule:machine:representation-violation-last}\ 
	&&RV(X) \impd U_2(X, Y) \wedge \D U_2(X, Z)
\end{align}
to detect representation violations.
For transition violations, we first define IDBs $NE_c^{\varepsilon}$, where $c \in {1, 2}$ stand for the respective counter and $\epsilon \in \{-1, 0, 1\}$ stands for the change of that counter value in a transition. Each $NE_c^{\varepsilon}$ detects situations when a correct transition was not executed, e.g. having $\epsilon = -1$, the timestamps $\ell$ and $\ell'$ that are marked by an outgoing $U_c$ in consecutive configurations satisfy $\ell - 1 > \ell'$, $\ell = \ell'$, or $\ell < \ell'$, which they should not. The rules are the following:
\begin{align}
    &NE_c^{-1}(Y) \impd U_c(Y, Z) \wedge U_c(X, Y) %\label{rule:machine:transition-violation-first}
    &&NE_c^{-1}(Y) \impd U_c(Y, Z) \wedge \D \D U_c(X, Y) \label{rule:machine:transition-violation2}\\
    &NE_c^{-1}(Y) \impd U_c(X, Y) \wedge \D U_c(Y, Z) %\label{rule:machine:transition-violation3}
    &&NE_c^{1}(Y) \impd U_c(Y, Z) \wedge U_c(X, Y) \label{rule:machine:transition-violation6}\\
	&NE_c^{1}(Y) \impd U_c(Y, Z) \wedge \D U_c(X, Y) %\label{rule:machine:transition-violation7}
	&&NE_c^{1}(Y) \impd U_c(X, Y) \wedge \D\D U_c(Y, Z) \label{rule:machine:transition-violation-last}\\
	&NE_c^{0}(Y) \impd U_c(Y, Z) \wedge \D U_c(X, Y)% \label{rule:machine:transition-violation4}
	&&NE_c^{0}(Y) \impd U_c(X, Y) \wedge \D U_c(Y, Z) \label{rule:machine:transition-violation5}
\end{align}
for $c \in \{1,2\}$. Then, again, any state can be inferred from a violation:
\begin{align}\label{rule:machine:tv-state}
	&S_i(X) \impd \Drefl NE_1^{\varepsilon_1}(Y) \wedge \D^\alpha U_1(X, Y) 
    &S_i(X) \impd \Drefl NE_2^{\varepsilon_2}(Y) \wedge \D^{\,\beta} U_2(X, Y)
%\label{rule:machine:tv-state2}
\end{align}
for each transition $\Theta(s_i, \alpha, \beta) = (s_j, \varepsilon_1, \varepsilon_2)$, and for all $\varepsilon_1, \varepsilon_2$ when $\Theta(s_i, \alpha, \beta)$ is not defined. Finally, we require%
\begin{align}\label{rule:machine:state-transition}
		&S_i(X) \impd T(X, Y) \wedge \D^\alpha U_1(X, Y) \land \D^{\,\beta} U_2(X, Y) \wedge S_j(Y),\\
        &\label{rule:machine:goal}
	G(X) \impd S_0(X) \wedge U_1(X, Y) \wedge U_2(X, Y).
\end{align}
This finalises the construction of $(\pi_M, G)$. Any expansion of $(\pi_M, G)$ starts with the body of the rule~\eqref{rule:machine:goal} and then continues by a sequence of bodies of the rules of the form \eqref{rule:machine:state-transition} and ends with a detection either of a representation violation, defined by \eqref{rule:machine:representation-violation-first}---\eqref{rule:machine:representation-violation-last}, or of a transition violation, defined by \eqref{rule:machine:tv-state}. If $M$ halts in $m$ steps, it is enough to consider expansions containing no more than $m$ bodies of \eqref{rule:machine:state-transition}. Thus, in this case $(\pi_M, G)$ is in \ACz. If, on the contrary, $M$ does not halt, we can use expansions representing arbitrarily long prefixes of its computation for a reduction from directed reachability problem, rendering the query \NLogSpace-hard.
\begin{restatable}{lemma}{LemmaHalting}\label{lm:machine:acz}
	If $M$ halts, $(\pi_M, G)$ is in \ACz. Otherwise, it is \NLogSpace-hard.
\end{restatable}
%
%A formal proof can be found in the appendix.

%%%%%%%%%%%%%%%%%%%%%%%%%%%%%%%%%%%%%%%%%%%%%%%%%%%%%%%%%%%%%%%%%%%%%%

% !TEX root =  ../main.tex

\section{Automata-Theoretic Tools for Queries with \Next} \label{sec:automata}

From now on, we focus on queries that use operators \Next and \Prev only. In this section, we develop a generalisation of the automata-theoretic approach to analysing query expansions proposed in \cite{Cosmadakis-Monadic-DLog-Boundedness-Decidable}. In Section \ref{sec:next} we use this approach to study the data complexity of this kind of queries.

Recall from Section~\ref{sec:preliminaries:mtdlog:expansions} the definitions of composition and expansion of rule bodies, alphabets $\Gamma_\pi^r$ and $\Gamma_\pi^i$, and the language $\expand(\pi, G) \sbs (\Gamma_\pi^r)^*(\Gamma_\pi^i)$. We observe that for any sequence of rule bodies $B_1, \dots, B_{n-1} \in \Gamma_\pi^r$ and $B_n \in \Gamma_\pi^i$
the compositions $B_1\circ\dots\circ B_{n - 1}$ and $B_1 \circ \dots\circ B_n$ are well-defined. They are words in the languages $(\Gamma_\pi^r)^*$ and $(\Gamma_\pi^r)^*(\Gamma_\pi^i)$, respectively. Note that $B_1 \circ \dots\circ B_n$ is a temporal CQ over schema \EDBpi, while $B_1\circ\dots\circ B_{n - 1}$ is that over $\EDBpi \cup \IDBpi$, since it contains the IDB atom of $B_{n - 1}$. %Slightly abusing the notation, we will regard any word $w$ from $(\Gamma_\pi^r)^*$ as a temporal CQ over \EDBpi obtained by omitting the IDB atom from $w$.
For $w \in (\Gamma_\pi^r)^*(\Gamma_\pi^i)$, $\Db_w$ is defined as the (temporal) database corresponding to (temporal) CQ $w$, while for $w \in (\Gamma_\pi^r)^*$ we define $\Db_w$ as the database corresponding to the CQ obtained from $w$ by omitting the IDB atom. For either $w$, $\Db_w$ is over the schema \EDBpi.
Having that, we define the language
$\accept(\pi, G) \sbs (\Gamma_\pi^r)^* \cup (\Gamma_\pi^r)^*(\Gamma_\pi^i)$ of all words $w \in (\Gamma_\pi^r)^* \cup (\Gamma_\pi^r)^*(\Gamma_\pi^i)$
%such that $w$, as a temporal CQ, is contained in $(\pi, G)$.
%
such that $\Db_w, \pi, 0 \mdl G(X)$.
%where $\Db_w$ is the temporal database corresponding to $w$ and $G(X)$ is then a ground predicate w.r.t. the database $\Db_w$.

Plain datalog queries are either in \ACz (called \textit{bounded}), or \LogSpace-hard (\textit{unbounded}), and a criterion of unboundedness can be formulated in language-theoretic terms~\cite{Cosmadakis-Monadic-DLog-Boundedness-Decidable}: a connected linear monadic plain datalog query $(\pi, G)$ is \emph{unbounded} if and only if for every $k$ there is $w \in \expand(\pi, G)$, $|w| > k$, such that its prefix of length $k$ is \textit{not} in $\accept(\pi, G)$.

\begin{example}\label{ex:pure-datalog-un-bounded}
	The query $(\pi, G)$, where $\pi$ is given by the following plain datalog rules, is unbounded.
	\begin{align}
		&G(X) \impd R(X, Y) \wedge S(Y, X) \wedge G(Y) \label{rule:pure-datalog-recursive}\\
		&G(X) \impd A(X) \label{rule:pure-datalog-initial-unbounded}
	\end{align}
	However, if we substitute the rule \eqref{rule:pure-datalog-initial-unbounded} with another initialisation rule
	\begin{align}
		&G(X) \impd S(X, Y) \wedge S(Y, Z) \label{rule:pure-datalog-initial-bounded}
	\end{align}
	it becomes bounded because for every $w \in \expand(\pi, G)$, $|w| > 2$ its prefix of length $2$ is in $\accept(\pi, G)$. To see this, consider the expansions of the query, i.e.,  $\expand(\{\eqref{rule:pure-datalog-recursive}, \eqref{rule:pure-datalog-initial-bounded}\}, G)$ given in Figure \ref{fig:pure-datalog}.
\end{example}
%
%\begin{figure}
%	\centering
%	\includegraphics[scale=0.27]{example-pure-datalog}
%	\caption{Every prefix of length $\leqslant$ 2 is in $\accept(\pi, G)$, by the mapping $h$ shown on the right.}
%	\label{fig:pure-datalog-unbounded}
%\end{figure}

\begin{figure}
	\centering
	\begin{subfigure}{.5\textwidth}
		\centering
		\includegraphics[height=3cm, keepaspectratio]{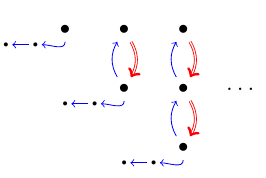}
		\caption{Expansions: {\color{red} double arrows} represent the rela-\\tion $R$, {\color{blue} single arrows}---the relation $S$.}
		\label{fig:virus:data}
	\end{subfigure}%
	\begin{subfigure}{.5\textwidth}
		\centering
		\includegraphics[height=3cm, keepaspectratio]{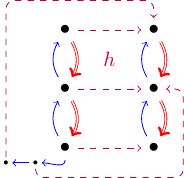}
		\caption{A homomorphism ({\color{purple} dashed arrows}) from an expansion of length 3 to its own prefix of length 2.}
		\label{fig:virus:reasoning}
	\end{subfigure}
	\caption{Illustrations for the query $\{\eqref{rule:pure-datalog-recursive}, \eqref{rule:pure-datalog-initial-bounded}\}, G)$ of Example \ref{ex:pure-datalog-un-bounded}.}
	\label{fig:pure-datalog}
\end{figure}

The authors of~\cite{Cosmadakis-Monadic-DLog-Boundedness-Decidable} construct finite state automata for languages $\expand(\pi, G)$ and $\notaccept(\pi, G)$, the complement of $\accept(\pi, G)$, and use them to check their criterion in polynomial space. Our goal is to generalise this technique to the temporalised case. However, we can not use finite automata to work directly with sequences of rule bodies, since the language $\accept(\pi, G)$, in the presence of time, may be non-regular, as demonstrated by the following example.
\begin{example}\label{ex:general-expansions}
	Consider a program
	\begin{align*}
		&G(X) \impd R(X, Y) \wedge G(Y)
		&&G(X) \impd A(X) \wedge \Next G(X)\\
		&G(X) \impd P(X) \wedge \Prev G(X)
		&&G(X) \impd P(X) \wedge R(Y, X).
	\end{align*}
	Denote its rule bodies $B_1(X, Y) = R(X, Y) \wedge G(Y)$, $B_2(X) = A(X) \wedge \Next G(X)$, and $B_3(X) = P(X) \wedge \Prev G(X)$, and the composition $w = B_1B_2B_2B_3B_3B_3$. The composition $w$ is in $\accept(\pi, G)$, and the corresponding database $\Db_w$ is given in Figure \ref{fig:general-expansions:a}.
	In general, a composition of the form $B_1B_2^nB_3^k$ is in $\accept(\pi, G)$ if and only if $n < k$, which, by a simple application of the pumping lemma, is a non-regular language.
\end{example}

%\begin{figure}
%	\centering
%	\includegraphics[height=3cm, keepaspectratio]{db-word-example-8.pdf}
%	\caption{The temporal database $\Db_w$  for $w = B_1B_2B_2B_3B_3B_3$ of Example \ref{ex:general-expansions}.}
%	\label{fig:general-expansions}
%\end{figure}

\begin{figure}
	\centering
	\begin{subfigure}{.5\textwidth}
		\centering
		\includegraphics[height=3cm, keepaspectratio]{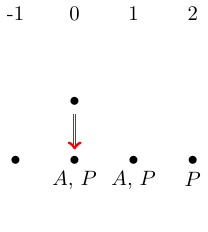}
		\caption{The temporal database $\Db_w$ for $w = B_1B_2B_2B_3B_3B_3$ of Example \ref{ex:general-expansions}.}
		\label{fig:general-expansions:a}
	\end{subfigure}%
	\begin{subfigure}{.5\textwidth}
		\centering
		\includegraphics[height=3cm, keepaspectratio]{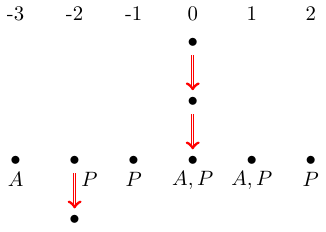}
		\caption{The temporal database $\Db_w$ for $w = B_1^2 B_2^2 B_3^5 B_2 B_1$ of Example \ref{ex:omega-descriptions}.}
		\label{fig:general-expansions:b}
	\end{subfigure}
	\caption{Expansions of \MTDLogX queries.}
	\label{fig:general-expansions}
\end{figure}

To overcome this, we introduce a larger alphabet and define more general versions of the languages $\expand(\pi, G)$ and $\accept(\pi, G)$ to regain their regularity.
Recall that every composition $w$ of rule bodies gives rise to a temporal database $\Db_w$. Instead of working with $w$, we use an exponentially larger alphabet $\Omega = 2^{\Gamma_\pi^r \cup \Gamma_\pi^i \cup \{\bot, \top\}}$ to describe $\Db_w$ as a word and define analogues of $\expand(\pi, G)$ and $\accept(\pi, G)$ over that alphabet. Consider a recursive rule
\begin{equation}\label{def:recursive-rule}
	D(X) \impd \Next^{k_1} R_1(\ovl{U}_1) \wedge \dots \wedge \Next^{k_s} R_s(\ovl{U}_s) \wedge \Next^k E(Y),
\end{equation}
where $E(Y)$ is the unique IDB atom in the rule body and $k_i, k \in \Zbb$. We call such a rule \textit{horizontal} if $X = Y$ and \textit{vertical} otherwise. In a composition $w \in (\Gamma^r_\pi)^* \cup (\Gamma^r_\pi)^*\Gamma_\pi^i$, a \textit{vertical (horizontal) segment} is a maximal subword that consists of vertical (respectively, horizontal) rule bodies.
For every composition $w$ we define the \textit{description} $[w] \in \Omega^*$ of the respective database $\Db_w$ as follows. Let $w = x_1y_1x_2y_2\dots x_ny_n$, where $x_i$ are vertical segments and $y_i$ are horizontal segments, with $x_1$ and $y_n$ possibly empty.
For each $x_i = B_1\dots B_n$ we set $[x_i]$ to be the sequence $\{B_1\}\dots\{B_n\}$ of singleton sets, each of which contains a vertical rule body. For each $y_i = B_1 \dots B_n$, we construct $[y_i]$ in $n$ steps, as follows.
Recall that $B_j$, $1 \leqslant j < n$,  has the form $A_j(X, \Ubf) \wedge \Next^{m_j}D(X)$, where $D(X)$ is the unique IDB atom in $B_j$.  %
Let $\ell_1 = 0$ and $\ell_j = \sum_{i  = 1}^{j-1} m_i$ for $j \leq n$. Intuitively, $\ell_j$ is the moment of time where the body $B_j$ lands in the composition $B_1\circ \dots \circ B_n$. Let $\ell'_1, \dots, \ell'_s$ be the ordering of the the numbers in the set $\{\ell_j\}_{j = 1}^n$ in the increasing order. We set, $\alpha_{\ell'_{k}}$ to be $\{B_j \mid \ell_j = \ell'_k \}$ for $\ell'_k \in \{\ell_1, \ell_n\}$; $\{ B_j \mid \ell_j = \ell_1 \} \cup \{\bot\}$ for $\ell'_k = \ell_1 \neq \ell_n$; $\{B_j \mid \ell_j = \ell_n\} \cup \{\top\}$ for $\ell'_k = \ell_n \neq \ell_1$; and $\{ B_j \mid \ell_j = \ell_1 \} \cup \{\top, \bot\}$ for $\ell'_k = \ell_1 = \ell_n$.
%
%We set for $1 \leq k \leq s$
%%
%$$
%\alpha_{\ell'_{k}} = \begin{cases}
%                  \{ B_j \mid \ell_j = \ell'_k \}, & \mbox{if } \ell'_k \notin \{\ell_1, \ell_n\} \\
%                  \{ B_j \mid \ell_j = \ell_1 \} \cup \{\bot\}, & \mbox{if }\ell'_k = \ell_1 \neq \ell_n \\
%                  \{ B_j \mid \ell_j = \ell_n \} \cup \{\top\}, & \mbox{if }\ell'_k = \ell_n \neq \ell_1 \\
%                  \{ B_j \mid \ell_j = \ell_1 \} \cup \{\top, \bot\}, & \mbox{if }\ell'_k = \ell_1 = \ell_n.
%                \end{cases}
%$$
%
Additionally, we set $\alpha_k = \emptyset$ for $k \in [\ell'_1, \ell'_s] \setminus \{ \ell'_1, \dots, \ell'_s \}$. Now we take $[y_i] = \alpha$.

% the composition $B_1 \circ \dots \circ B_n$ has the form $\Next^{-l} \Abf_{-l} \wedge \Next^{-l + 1} \Abf_{-l + 1} \wedge\dots\wedge \Abf_0 \wedge\dots\wedge \Next^{r}\Abf_r$, where each $\Abf_i$ is a conjunction of horizontal rule bodies.
%
%Viewing each $\Abf_i$ as a set of its conjuncts, we set $[y_i]$ to be the sequence $\Abf_{-l} \dots \bot\, \Abf_0 \dots \top\,\Abf_j\dots \Abf_r$ where $\bot$ comes just before $\Abf_0$, and $\top$ comes just before that particular conjunction $\Abf_j$ that contains $B_n$. If $j = 0$, we allow the sequence $\bot\top\Abf_0$.

Finally, $[w] = [x_1][y_1]\dots [x_n][y_n]$. We use the symbol $\Lambda$ to refer to letters of $[w]$. We call letters of the form $\{B\}$, where $B$ is a vertical rule body, \textit{vertical letters}, and the rest---\textit{horizontal letters}. Consequently, we can speak of vertical and horizontal segments of $[w]$, meaning maximal segments composed of vertical (respectively, horizontal) letters only.

Intuitively, $[w]$ describes $\Db_w$, which can be seen as composed of $\Db_{x_1}, \Db_{y_1}, \dots, \Db_{x_n}, \Db_{y_n}$, described by $[x_1], [y_1], \dots, [x_n], [y_n]$, respectively. The symbol $\bot$, representing a vertical line meeting a horizontal one, marks the point in time where $\Db_{x_i}$ is connected to $\Db_{y_i}$, while the symbol $\top$, analogously, shows where $\Db_{y_i}$ is connected to $\Db_{x_{i + 1}}$.
\begin{example}\label{ex:omega-descriptions}
	Recall the rule bodies of Example \ref{ex:general-expansions} and consider the composition $w = B_1^2 B_2^2 B_3^5 B_2 B_1$. The corresponding $\Db_w$ is depicted in Figure \ref{fig:general-expansions:b}.
	Then $x_1 = B_1^2$, $y_1 = B_2^2 B_3^5 B_2$, $x_2 = B_1$ and $y_2$ is empty.
%	Moreover, $y_1$ gives the composition $$\Next^{-3} B_2 \wedge \Next^{-2}\underline{B_3} \wedge \Next^{-1} B_3 \wedge [B_2 \wedge B_3] \wedge \Next [B_2 \wedge B_3] \wedge \Next^2 B_3,$$ where the underlined $B_3$ corresponds to the last letter of $y_1$.
	Thus, $[w]$ is equal to
	\begin{equation}\label{eq:w-description}
		\{B_1\}\{B_1\}\{B_2\}\{\top, B_3\}\{B_3\}\{\bot,B_2, B_3\}\{B_2, B_3\}\{B_3\}\{B_1\}
	\end{equation}
\end{example}

%\begin{figure}
%	\centering
%	\includegraphics[height=5cm, keepaspectratio]{db-word-example-9.pdf}
%	\caption{The temporal database $\Db_w$ of Example \ref{ex:omega-descriptions}.}
%	\label{fig:omega-descriptions}
%\end{figure}

Not every word over $\Omega$ correctly describes a temporal database. This motivates the following definition: $\alpha \in \Omega^*$ is \textit{correct} if (i) every symbol $\Lambda$ is either a singleton (standing for a vertical rule body) or a set of horizontal rule bodies, with a possible addition of $\bot, \top$, and (ii) every horizontal segment of $\alpha$ preceded by a vertical segment contains exactly one $\bot$, and every horizontal segment followed by a vertical one---exactly one $\top$. A correct word $\alpha \in \Omega^*$ describes a temporal database $\Db_\alpha$ similarly to how $[w]$ describes $\Db_w$. Formally, break $\alpha$ into vertical/horizontal segments $\chi_1\upsilon_1\dots\chi_n\upsilon_n$. For a vertical segment $\chi_i = \{B_1\}\dots\{B_n\}$, set $Q_{\chi_i} = B_1 \circ \dots \circ B_n$. For a horizontal segment $\upsilon_i = \Lambda_1 \dots \Lambda_s$, let $\Lambda_{j_\bot}$ be the one containing $\bot$ and $\Lambda_{j_\top}$---containing $\top$. Then $Q_{\upsilon_i}$ is the conjunction of rule bodies from $\upsilon_i$, where each $B \in \Lambda_j$ is prefixed by $\Next^{j - j_\bot}$, plus, if $i \neq n$, an IDB atom $\Next^{j_\top - j_\bot}D(X)$. Finally, set $\Db_\alpha = \Db_{Q_{\alpha}}$ for $Q_{\alpha} = Q_{\chi_1} \circ Q_{\upsilon_1} \circ \dots \circ Q_{\chi_n} \circ Q_{\upsilon_n}$. This $Q_\alpha$ will be also useful further.

We are now ready to define languages over $\Omega$ that will be useful to study the data complexity of our queries. Let $\acceptom(\pi, G)$ be the language of correct words $\alpha$ such that $\Db_\alpha, \pi, 0 \mdl G(X)$, with $X\in\Delta_{\Db_\alpha}$, and $\notacceptom(\pi, G)$ be its complement. We need to define the language of expansions over the alphabet $\Omega$ that we will use together with the language $\notacceptom(\pi, G)$ to formulate a criterion for \LogSpace-hardness similar to the one of \cite{Cosmadakis-Monadic-DLog-Boundedness-Decidable}.
It would be natural to take the language of expansions as $\{[w] \mid w \in \expand(\pi, G)\}$. However, it is harder to define an automaton recognising such a language than it is for the language of $[w]$s as above where each horizontal letter $[w]_i$ may be extended (as a set) with arbitrary ``redundant'' rule bodies, and each horizontal segment may be extended, from the left and right, by ``redundant'' horizontal letters. For example, we will include into the language of expansions the word $\{B_1\}\{B_1\}\{B_3\}\{B_2, B_3\}\{\top, B_3\}\{B_2, B_3\}\{\bot,B_2, B_3\}\{B_2, B_3\}\{B_3\}\{B_1\}$ alongside~\eqref{eq:w-description}. It turns out that the latter language works for our required criteria as good as the former one. Formally, if $\alpha, \beta \in \Omega^*$, we write $\alpha \preccurlyeq \beta$ if
$\alpha = x_1y_1\dots x_ny_n$ and $\beta = x_1y'_1\dots x_ny'_n$, where $x_i$ are vertical segments and $y_i, y'_i$ are horizontal segments, and if $y_i = a_1\dots a_m$, $y'_i = b_1\dots b_s$, then there is $k \geqslant 0$ such that $m + k \leqslant s$ and $a_j \sbs b_{j + k}$, $1 \leqslant j \leqslant m$. We define $\expandom(\pi, G)$ to be the set of correct words $\alpha \in \Omega^*$ such that $[w] \preccurlyeq \alpha$ for some $w \in \expand(\pi, G)$.

%
%\textcolor{red}{Let\nb{V: old} us go back to the expansion $w$ of Example \ref{ex:omega-descriptions} and the respective $[w]$ in \eqref{eq:w-description}. Let $\alpha$ coincide with $[w]$ except for the third letter, being $\{B_2, B_3\}$ instead of just $\{B_2\}$. Distinguishing $\alpha$ from $[w]$ would require to somehow notice the redundant rule body $B_3$. For our purposes, however, that distinction is not important, hence the following definition.
%%
%Given $\alpha, \beta \in \Omega^*$, we write $\alpha \preccurlyeq \beta$ if $\alpha = \Lambda^a_1\dots \Lambda^a_m$, $\beta = \Lambda^b_1\dots \Lambda^b_m$, and if $\Lambda^a_i$ is a horizontal letter then $\Lambda^a_i \sbs \Lambda^b_i$, otherwise $\Lambda^a_i = \Lambda^b_i$. We define $\expandom(\pi, G)$ to be the set of $\alpha \in \Omega^*$ such that $[w] \preccurlyeq \alpha$ for some $w \in \expand(\pi, G)$.}

It is important for our purposes that these languages are regular. For $\Expand(\pi, G)$, the rules of the program $\pi$ may be naturally seen as transition rules of a two-way automaton whose states are the IDBs of $\pi$ (plus a final state). The initial state is the one associated with $G$, and the final state is reached by an application of an initialisation rule.
%
%This polynomial-sized machine can be converted to a one-way nondeterministic automaton with an exponential blow-up, using the method of crossing sequences~\cite{Shepherdson-2way}, which guarantees $|\pi|$-manipulatability.
%
%Having complexity considerations in mind, we also demonstrate that the respective automata have some nice algorithmic properties. Namely, we say that an automaton $\A = (S, \aleph, s^0, T, F)$, with the set of states $S$, input alphabet $\aleph$, initial state $s^0$, transition relation $T$, and the set of final states $F$, is $n$-\textit{manipulatable}, for $n \in \N$, if there is a polynomial $p(x)$ and one-to-one encodings $u \colon S \to \{0, 1\}^{p(n)}$ and $v \colon \aleph \to \{0, 1\}^{p(n)}$, such that:
%%
%\begin{enumerate}
%	\item[]\textbf{(repr)} the sets $u(S), \{u(s^0)\}$, $u(F)$, and $v(\aleph)$ are decidable\nb{V: what is a set decidable in \PSpace?} in \PSpace;
%	
%	\item[]\textbf{(tran)} the set $\{(u(s), v(\Lambda), u(s')) \mid (s, \Lambda, s') \in T\}$ is decidable in \PSpace.
%\end{enumerate}
%%
%Note that it follows, from the fact that $u$ is one-to-one, that $|S| \leqslant 2^{p(n)}$.
%%
%We observe the following.
%
\begin{restatable}{lemma}{LemmaExpand}\label{lm:body-language:expand}
	For any \lmdln-query $(\pi, G)$, the language $\Expand(\pi, G)$ is regular.
\end{restatable}
The case of $\notacceptom(\pi, G)$ is more involved. Generalising from~\cite{Cosmadakis-Monadic-DLog-Boundedness-Decidable}, for an automaton to recognise if $\alpha \in \notacceptom(\pi, G)$ it suffices to guess (by nondeterminism) an enrichment $\E$ of $\Db_\alpha$, and check that $\E$ is a model of $\pi$ and that $\E$ does not contain the atom $G(X)$. Moreover, since $\pi$ is connected, for $\E$ to be a model it is enough that every piece of $E$ of radius $|\pi|$, in terms of Gaifman graph, satisfies all the rules. The idea is to precompute the answer for each such piece and encode them in the state-space of the automaton.
The problem, specific to the temporalised case, is that enrichments are infinite in the temporal dimension. To resolve this, we observe that there are still finitely many EDB atoms in $\E$. Since, once again, $\pi$ is connected, to check that rules with EDBs are satisfied it is enough to consider only those pieces of $\E$ that contain an EDB atom. For the rest of the rules, it suffices to check if such a finite piece \textit{can} be extended into an infinitely in time to give a model of $\pi$. The rules without EDBs can be seen as plain \LTL rules, so we can employ satisfiability checking for \LTL to perform this check. 
%The details can be found in the appendix.
%
%
%In the atemporal case, such an enrichment can be described by a word in an \textit{enriched alphabet}, by adding to each letter of $\alpha$ various IDB atoms. Then checking that such an enriched word defines a model of $\pi$, for connected programs, boils down to checking that every subword of it of length $|\pi|$ satisfies all rules of $\pi$. This is done in~\cite{Cosmadakis-Monadic-DLog-Boundedness-Decidable} by memorising in the automaton\nb{V: what is memorising in the automaton?} all subwords of length $|\pi|$ that have this property.
%
\begin{restatable}{lemma}{LemmaNoAccept}\label{lm:body-language:notaccept}
	For any \lmdln-query $(\pi, G)$, the language $\notacceptom(\pi, G)$ is regular.
	%recognised by a nondeterministic finite state automaton with $2^{\poly{|\pi|}}$ states.
	%which is $|\pi|$-manipulatable.
\end{restatable}
%
%
%%
%In the temporal case, $\E$ is infinite (in the temporal dimension). We circumvent that using the notion of an $n$-cut --- a finite temporal database, denoted $\E|_n$ and obtained from $\E$ by discarding, for every object $d$ of $\E$, all (IDB) atoms that appear in $\E$ with timestamps that are distant by more than $n$ from any timestamp with an EDB atom involving $d$. For $n = |\pi|$, $n$-cuts of enrichments of $\Db$ are representable by enriched words (where the letters of $\alpha$ are enriched by various \textit{temporalised} IDB atoms of the form $\Next^k D(Y)$, where $|k| < n$). Note that a rule having an EDB atom in its body can not affect an IDB\nb{V: what is EDB atom affectingis IDB atom?} in $\E$ distant by more than $|\pi|$ from all EDB atoms. For every object $d$ of $\E|_n$ call the set of IDB atoms involving $d$ with their timestamps its \textit{timeline}. Then, given $\E|_n$, to decide if $\E$ is a model of $\pi$ it is enough to check two conditions: that $\E|_n$ satisfies all rules of $\pi$ as a temporal database\nb{V: how $A \leftarrow \Next A$ can be satisfied by in a finite db?} and that for every $d$ its timeline can be extended to a model of horizontal rules of $\pi$. The first condition can be checked in the same fashion as in the pure datalog case, and for the second condition we use a standard polynomial-space procedure for \LTL.
%%
%
\begin{corollary}\label{cor:body-language:acceptom}
	For any \lmdln-query $(\pi, G)$, the language $\acceptom(\pi, G)$ is regular.
\end{corollary}
%

%%%%%%%%%%%%%%%%%%%%%%%%%%%%%%%%%%%%%%%%%%%%%%%%%%%%%%%%%%%%%%%%%%%%%%

% !TEX root =  ../main.tex

\section{Decidability for Connected Linear Queries with \Next}\label{sec:next}

We use the automata introduced in the previous section to prove a positive result.
\begin{theorem}\label{thm:next-and-prev}
	(i) Every connected \lmdln query is either in \ACz, or in $\ACCz\setminus\ACz$, or \NCo-complete, or \LogSpace-hard. (ii) It is \PSpace-complete to check whether such a query is \LogSpace-hard; whether it belongs to \ACz, $\ACCz$, or is \NCo-complete can be decided in \ExpSpace.
\end{theorem}

We first deal with \emph{(i)}. Intuitively, \LogSpace-hardness is a consequence of the growth of query expansions in the relational domain. If this growth is limited, the query essentially defines a certain temporal property, which can be checked in \NCo. Formally, given a word $\alpha \in \Omega^*$, we define the $\height{\alpha}$ as the number of vertical letters in~$\alpha$. Then, we call a query \textit{vertically unbounded} if for every $k$ there is a word $\alpha \in \Accept(\pi, G)$, $\height{\alpha} > k$, such that every prefix of $\alpha$ of height $k$ is in $\notacceptom(\pi, G)$. Otherwise, the query is called \textit{vertically bounded}.

Vertically unbounded queries can be shown to be \LogSpace-hard by a direct reduction from the undirected reachability problem. Namely, take the deterministic automata for $\notacceptom(\pi, G)$ and $\acceptom(\pi, G)$, supplied by Lemma \ref{lm:body-language:notaccept} and Corollary \ref{cor:body-language:acceptom}, and apply the pumping lemma to obtain words $\xi, \upsilon, \zeta, \gamma$ such that $\height{\upsilon} > 0$, $\xi\upsilon^i\zeta \in \notacceptom(\pi, G)$ and $\xi\upsilon^i\zeta\gamma \in \acceptom(\pi, G)$, for all $i \geqslant 0$. Then, given a graph $\Gmc$ and two nodes $s, t$, use copies of $\Db_\upsilon$ to simulate the edges of $\Gmc$, and attach $\Db_{\xi}$ to $s$ and $\Db_{\zeta\gamma}$ to $t$. Thus, you will obtain a temporal database $\Db_\Gmc$, where $\Db_\Gmc, \pi, 0 \mdl G(s)$ if and only if there is a path from $s$ to $t$.
\begin{restatable}{lemma}{LmVerticallyUnboundedMeansLogSpaceHard}\label{lm:verticaly-unbounded-means-logspace-hard}
	If a query is vertically unbounded, then it is \LogSpace-hard.
\end{restatable}
Observe that $\Db, \pi, \ell \mdl G(d)$ whenever $\Db, \ell \mdl Q_\alpha(d)$ for some $\alpha \in \Accept(\pi, G)$. If $\Db, \ell \mdl Q_\alpha(d)$, then $x_1, y_1, \dots, y_n, x_n$, the vertical and horizontal segments of $\alpha$, appear in \Db, starting from $d$ at time $\ell$.
Expand $y_i$ to $y'_i$, the $\preccurlyeq$-maximal horizontal segment fitting in \Db, and let $\beta = x_1y'_1\dots x_ny'_n$. Then $\beta \in \Accept(\pi, G)$ and $\Db, \ell \mdl Q_\beta(d)$. Thus, to check $\Db, \pi, \ell \mdl G(d)$, it suffices to find vertical segments of some $\alpha \in \Accept(\pi, G)$, taking $\preccurlyeq$-maximal horizontal segments, and check that $\beta \in \Accept(\pi, G)$. If $(\pi, G)$ is vertically bounded, there are finitely many vertical segments of interest. Finding vertical segments, as well as extracting $\preccurlyeq$-maximal horizontal ones, can be done by an \ACz circuit.
\begin{restatable}{lemma}{LmVerticallyBoundedMeansACz}\label{lm:lower-classes}
	If $(\pi, G)$ is vertically bounded, then the data complexity of $(\pi, G)$ coincides with that of checking membership in $\acceptom(\pi, G)$, modulo reductions computable in \ACz.
\end{restatable}
Recall that $\acceptom(\pi, G)$ is a regular language, so checking membership in it is either in \ACz, or $\ACCz\setminus\ACz$, or \NCo-complete \cite{Straubing94}. This settles the part \emph{(i)} of Theorem \ref{thm:next-and-prev}: a query is either vertically unbounded and thus \LogSpace-hard, or vertically bounded and thus belongs to one of the three classes mentioned above.

It remains to address the part \emph{(ii)} of Theorem \ref{thm:next-and-prev}. 
%Obviously, the complexity of the decision problem for $\acceptom(\pi, G)$ coincides with that of $\notacceptom(\pi, G)$. 
By a careful analysis of the proof of Lemma \ref{lm:body-language:notaccept} one can show that $\notacceptom(\pi, G)$ is recognised by a nondeterministic automaton of size $2^{\poly{|\pi|}}$. Further, checking whether its language belongs to $\ACz$, $\ACCz\setminus\ACz$, or is \NCo-complete, can be done via the polynomial-space procedure developed in \cite{LTL-Rewritability-Checking}. Thus, given a vertically bounded query, its data complexity can be established in exponential space.

For the vertical boundedness itself, note that substituting $\Accept(\pi, G)$ with $\Expand(\pi, G)$ in the respective definition preserves all the results proven so far. For $\Expand(\pi, G)$, we can get from Lemma \ref{lm:body-language:expand} an exponential-size one-way automaton. Checking that the query is vertically unbounded can be done in nondeterministic space logarithmic in the size of the automata for $\Expand(\pi, G)$ and $\Notaccept(\pi, G)$, as it boils down to checking reachability in their Cartesian product. 
%We leave the details for the appendix.

\begin{restatable}{lemma}{LmVerticallyBoundedChecking}\label{lm:vertical-boundedness:checking}
	Checking if a connected \lmdln query is vertically bounded is in \PSpace.
\end{restatable}

For the lower bound, we show that checking boundedness is already \PSpace-hard for connected linear monadic queries in plain datalog. In fact, \PSpace-hardness was proved in \cite{Cosmadakis-Monadic-DLog-Boundedness-Decidable} for \textit{program boundedness} of disconnected programs. A program $\pi$ is called bounded if $(\pi, G)$ is bounded for every IDB $G$ in $\pi$. We were able to regain the connectedness of $\pi$ by focusing on query boundedness (also called \textit{predicate boundedness}) instead. The idea combines that of \cite{Cosmadakis-Monadic-DLog-Boundedness-Decidable} with that of Section \ref{sec:diamond-undecidability}: define an IDB $F$ that slides along a computation, this time of a space-bounded Turing machine, looking for an erroneous transition.

\begin{restatable}{lemma}{LmBoundednessPSpaceHard}\label{lm:boundedness:hardness}
	Deciding boundedness of connected linear monadic queries in plain datalog is \PSpace-hard.
\end{restatable}

Lemmas \ref{lm:verticaly-unbounded-means-logspace-hard} and \ref{lm:lower-classes} bring us to an interesting consideration: to understand the data complexity of a given query one should analyse its behaviour in the relational domain separately from that in the temporal domain. This can be given a precise sense as follows.
\begin{restatable}{proposition}{PropDecomposition}\label{prop:decomposition}
	For every connected \lmdln query $(\pi, G)$ there exist a plain datalog query $(\pi_d, G)$ and a plain \LTL query $(\pi_t, G)$, such that:
	\begin{enumerate}
		\item $(\pi, G)$ is vertically bounded if and only if $(\pi_d, G)$ is bounded;
		\item if $(\pi, G)$ is vertically bounded then its data complexity coincides with that of $(\pi_t, G)$.
	\end{enumerate}
\end{restatable}

This highlights the `weak' kind of interaction between the two domains provided by the operators \Next and \Prev. Given the undecidability results of Section \ref{sec:diamond-undecidability}, no such or similar decomposition is possible for queries with \Df.

Technically, both $\pi_d$ and $\pi_t$ are obtained by simulating the deterministic automaton $\Apig$ for the language $\acceptom(\pi, G)$ provided by Corollary \ref{cor:body-language:acceptom}. In both programs, the IDBs correspond to the states of \Apig and EDBs to the letters of $\Omega$. For $\pi_t$ these EDBs are unary and all the rules are horizontal, so that the expansions unwind fully in the temporal domain. In the case of $\pi_d$, the EDBs are binary and every vertical transition of \Apig is a step by a binary relation, while the horizontal transitions are skipped (thus, vertical boundedness becomes just boundedness). In both programs, initialisation rules correspond to \Apig reaching an accepting state.

We conclude this section with a consideration on disconnected queries. In \cite{Cosmadakis-Monadic-DLog-Boundedness-Decidable}, deciding boundedness is based on the fact that $\accept(\pi, G)$ remains regular even when $\pi$ has disconnected rules. This is not the case for \lmdln, as can be seen from the following example.
\begin{example}\label{ex:disconnected}
	Consider the program $\pi$ of four rules:
\begin{align*}%\label{example:disconnected}
	&G(X) \impd A(X) \wedge \Next D(X)
	&&D(X) \impd \Next D(X)\\
	&D(X) \impd R(X, Y) \wedge \Prev B(Y) \wedge \Prev D(Y)
	&&D(X) \impd A(Y) \wedge B(X)
\end{align*}
Let $B_1 = A(X) \wedge \Next D(X)$, $B_2 = \Next D(X)$, and $B_3 = R(X, Y) \wedge \Prev B(Y) \wedge \Prev D(Y)$. Then $B_1B_2^nB_3^m \in \accept(\pi, G)$, and, consequently, $\{\bot, B_1\}\{B_2\}^{n}\{\top\}\{B_3\}^{m} \in \Accept(\pi, G)$, whenever $m > n \geqslant 1$. This property is not recognisable by any finite state automaton. For more general classes of automata suitable to recognise $\accept(\pi, G)$ or $\Accept(\pi, G)$ in the disconnected setting, the properties that are to be checked along the lines of \cite{Cosmadakis-Monadic-DLog-Boundedness-Decidable}, such as language emptiness or finiteness, become undecidable. Therefore, disconnected queries possibly require a different approach for analysing the data complexity.
\end{example}

%%%%%%%%%%%%%%%%%%%%%%%%%%%%%%%%%%%%%%%%%%%%%%%%%%%%%%%%%%%%%%%%%%%%%%

% !TEX root =  ../main.tex

\section{Conclusions and Future Work}\label{sec:conclusions}

We have started investigating the complexity of determining the data complexity of answering monadic datalog queries with temporal operators.
For linear connected queries with operators \Next/\Prev, we have generalised the automata-theoretic technique of \cite{Cosmadakis-Monadic-DLog-Boundedness-Decidable}, developed originally for plain datalog, to establish an $\ACz/\ACCz /\NCo/\LogSpace/\NLogSpace$ classification of temporal query answering and proved that deciding \LogSpace-hardness of a given query is \PSpace-complete, while checking its membership in \ACz or \ACCz can be done in \ExpSpace. As a minor side product, we have established \PSpace-hardness of deciding boundedness of atemporal connected linear monadic datalog queries.
Rather surprisingly and in sharp contrast to the \Next/\Prev case, it turns out that checking (non-trivial) membership of queries with operators \Df/\Dp in the above complexity classes is undecidable. The results of this paper lead to a plethora of natural and intriguing open questions. Some of them are briefly discussed below.
\begin{enumerate}
	\item What happens if we disallow applications of $\Df/\Dp$ to binary EDB predicates in \lmdld-queries? We conjecture that this restriction  makes checking membership in the above complexity classes decidable. In fact, this conjecture follows from a positive answer to the next question.
	
	\item Can our decidability results for \lmdln be lifted to \mdln-queries? Dropping the linearity restriction in the atemporal case results in the extra data complexity class, \PTime, and the higher complexity, 2\ExpTime-completeness, of deciding boundedness. The upper bound was obtained using tree automata in \cite{Cosmadakis-Monadic-DLog-Boundedness-Decidable}, and we believe that this approach can be generalised to connected \mdln-queries in a way similar to what we have done above.
	
	\item On the other hand, dropping the connectedness restriction might turn out to be trickier, if at all possible, as shown by Example \ref{ex:disconnected}. Finding a new automata-theoretic characterisation for disconnected \lmdln-queries remains a challenging open problem.
	
	\item A decisive step in understanding the data complexity of answering queries mediated by a description logic ontology and monadic disjunctive datalog queries was made in~\cite{DBLP:journals/tods/BienvenuCLW14,DBLP:journals/lmcs/FeierKL19} by establishing a close connection with constraint satisfaction problems (CSPs). In our case,  quantified CSPs (see, e.g.,~\cite{DBLP:conf/focs/Zhuk24}) seem to be more appropriate. Connecting the two areas might be beneficial to both of them.

    \item In the context of streaming data, it would be interesting to  investigate the data complexity classes and the complexity of recognising them for datalog\textsl{MTL}-queries~\cite{DBLP:journals/jair/BrandtKRXZ18,DBLP:conf/ijcai/RyzhikovWZ19,DBLP:conf/aaai/WangHWG22}.
\end{enumerate}

%Passing\nz{tbd} from standard finite automata to tree automata, the authors of \cite{Cosmadakis-Monadic-DLog-Boundedness-Decidable} extended their analysis to nonlinear queries and established 2\ExpTime-completeness of boundedness checking. We believe that the same extension can work for our temporalised language \mdln. Dropping the linearity restriction will introduce a new possible class for data complexity, \PTime, and require a careful analysis of the temporal properties of tree-shaped query expansions.

%At the same time, we have demonstrated by an example that the behaviour of disconnected queries can not be characterised by automata in the same fashion as it is done in the atemporal case.
%%
%The existence of a different characterisation for disconnected queries suitable to analyse their data complexity is thus an intriguing open question.
%
%\textcolor{red}{non-linear monadic datalog is harder: new class P, boundedness is 2\ExpTime-complete, \ExpSpace for LTL; what is known about other classes - not much I guess? (Lutz? CSP? not clear how to use in in the temporal case)}

\bibliographystyle{plainurl}% the mandatory bibstyle
\bibliography{main}

\appendix

% !TeX spellcheck = en_GB

\section{Proofs for Section \ref{sec:preliminaries}} \label{sec:appendix}

%% \setcounter{theorem}{1}
% \begin{proposition}\label{app:prop:preliminaries:tcq-complexity}
% 	Given a temporal CQ $Q(X_1, \dots, X_k)$, checking $\Db, \ell \mdl Q(d_1, \dots, d_k)$ is in \ACz for data complexity.
% \end{proposition}
%
\PropACz*
\begin{proof}
	Let $\Sigma$ be a database schema and $\Db$---a temporal database over $\Sigma$. Let $\Sigma' = \{R^{m + 1} | R^m \in \Sigma\}$, where the superscript $m$ denotes the arity of the relation symbol $R$. Then $\Smf_\Db$ is a two-sorted first-order structure over the signature $\Sigma' \cup \{<\}$ with the domain $\Delta_\Db \cup \tem(\Db)$, and such that $R(\bar d, \ell)$ is true in $\Smf_\Db$ if and only if $\bar d \in \Delta_\Db^m, \ell \in \tem(\Db)$ and $\Db, \ell \mdl R(\bar d)$. Moreover, $(\ell < \ell')$ is true in $\Smf_\Db$ whenever $\ell, \ell' \in \tem(\Db)$ and $\ell < \ell'$ in \Zbb. The truth relation $\mdl$ for two-sorted first-order formulas over $\Smf_\Db$ are defined in the expected way (see \cite{LTL-Rewritability} for details).
	
	 We show that there exists an \FOLess-formula $\psi_Q(\Xol, t)$ such that for any \Db, any $\bar d \in \Delta_\Db^{|\Xol|}$, and any $\ell \in \tem(\Db)$, $\Db, \ell \models Q(\bar d)$ if and only if $\psi_Q(\bar d, \ell)$ is true in $\mathfrak S_\Db$. For simplicity, we slightly abuse the notation and write $\Db \mdl \psi(\bar d, \ell)$. We first prove two auxiliary lemmas.
	
	\begin{lemma}\label{app:lm:loop}
		If $\p(\Zol)$ is defined by \eqref{bnf:preliminaries:tcq} and has $N$ temporal operators, then for any \Db with $\tem(\Db) = [l, r]$ and any $\bar d \in \Delta_\Db^{|\Zol|}$:
		\begin{align}
			&\Db, r + N + 1 \mdl \p(\bar d) \iff \Db, \ell \mdl \p(\bar d), \text{ for all $\ell > r + N$} \label{app:eq:loop:first}\\
			&\Db, l - N - 1 \mdl \p(\bar d) \iff \Db, \ell \mdl \p(\bar d), \text{ for all $\ell < l - N$} \label{app:eq:loop:second}
		\end{align}
		\begin{proof}
			The proof is by induction. For $N = 0$, $\p$ is a conjunction of atoms. Clearly, $\Db, \ell \not\mdl \p(\bar d)$ for any $\ell \notin [l, r]$. We justify the induction step for \eqref{app:eq:loop:first}, while the case of \eqref{app:eq:loop:second} is analogous. We do another induction on the structure of the formula. The case $\p = R(Z_1, \dots, Z_m)$ is trivial. In the case $\p = \p_1 \wedge \p_2$ by the induction hypothesis on $N$ everything works. For the temporal operators, we consider the most interesting case of $\p = \Dp \p_1$, while the other cases can be proved by a simple application of the induction hypothesis and are left to the reader.
			Let $\p = \Dp \p_1$. If $\Db, \ell \not\mdl \Dp\p_1(\bar d)$ for all $\ell > r + N$, we are done. Otherwise, let $\ell_0 > r + N$ be the least timestamp such that $\Db, \ell \mdl \Dp\p_1(\bar d)$. By the choice of $\ell_0$, $\Db, \ell_0 - 1 \mdl \p_1(\bar d)$, where $\ell_0 - 1 \geqslant r + N$. By the induction hypothesis, $\Db, r + N \mdl \p_1(\bar d)$, thus $\Db, r + N \mdl \p_1(\bar d)$, and therefore $\Db, \ell \mdl \Dp\p_1(\bar d)$ for $\ell > r + N$.
		\end{proof}
	\end{lemma}
	
	\begin{lemma}\label{app:lm:rewriting}
		Let $\p(\Zol)$ be defined by \eqref{bnf:preliminaries:tcq}. For any subformula $\varkappa(\Zol)$ of $\varphi$, there exist formulas $\psi_\varkappa(\Zol, t)$  and $\psi_\varkappa^{i}(\Zol)$ for $i \in [-N - 1, \dots, -1] \cup [1, \dots, N + 1]$,
		such that for any \Db with $\tem(\Db) = [l, r]$ and any objects $\bar d \in \Delta_\Db^{|\Zol|}$ we have:
		\begin{align}
			&\Db, \ell \models \varkappa(\bar d) \iff \Db \models \psi_\varkappa(\bar d, \ell), \text{ for } l \leqslant \ell \leqslant r\\
			&\Db, (r + i) \models \varkappa(\bar d) \iff \Db \models \psi_\varkappa^i(\bar d), \text{ for } 1 \leqslant i \leqslant N + 1\\
			&\Db, (l + i) \models \varkappa(\bar d) \iff \Db \models \psi_\varkappa^i(\bar d), \text{ for } -N - 1 \leqslant i \leqslant 1
		\end{align}
	\end{lemma}
	\begin{proof}
		The proof is, once again, by induction. For the base case, we set $\psi_R(\Zol, t) = R(\Zol, t)$ and $\psi_R^i(\Zol) = \bot$. For an induction step, let $\min$ be defined by $\forall t (\min \leqslant t)$ and $\max$ by $\forall t (t \leqslant \max)$, and the predicate $t' = t + 1$ by $(t < t') \wedge \forall \tau . [(\tau < t) \wedge (\tau < t')] \vee [(t < \tau) \wedge (t' < \tau)]$. Then we define:
		\begin{align}
			&\psi_{\varkappa_1 \wedge \varkappa_2} = \psi_{\varkappa_1} \wedge \psi_{\varkappa_2}\\
			&\psi^i_{\varkappa_1 \wedge \varkappa_2} = \psi^i_{\varkappa_1} \wedge \psi^i_{\varkappa_2}, \text{ for all } i
		\end{align}
		\begin{align}
			&\psi_{\Next \varkappa}^{-1}(\Zol) = \psi_{\varkappa}(\Zol, \min)\\
			&\psi_{\Next \varkappa}^{N + 1}(\Zol) = \psi^{N + 1}_{\varkappa}(\Zol)\\
			&\psi_{\Next \varkappa}^i(\Zol) = \psi_{\varkappa}^{i+1}(\Zol), \text{ for all other } i\\
			&\psi_{\Next \varkappa}(\Zol, t) = \exists t' [(t' = t + 1) \land \psi_{\varkappa}(\Zol, t')] \lor ((t = \max) \land \psi_{\varkappa}^1(\Zol))
		\end{align}
		\begin{align}
			&\psi_{\D \varkappa}^{i}(\Zol) = \bigvee_{i < j \leqslant N + 1} \psi^j_{\varkappa}(\Zol),  \text{ for } 0 < i < N + 1\\
			&\psi_{\D \varkappa}^{i}(\Zol) = \bigvee_{i < j \leqslant N + 1} \psi^j_{\varkappa}(\Zol) \wedge \exists t [(\min \leqslant t \leqslant \max) \wedge \psi_\varkappa(\Zol, t)],  \text{ for } -N - 1 \leqslant i < 0\\
			&\psi_{\D \varkappa}^{N + 1}(\Zol) = \psi^{N + 1}_{\varkappa}(\Zol)\\
			&\psi_{\Next \varkappa}(\Zol, t) = \exists t' [(t < t') \land \psi_{\varkappa}(\Zol, t')] \lor \bigvee_{1 \leqslant i \leqslant N + 1} \psi^{i}_{\varkappa}(\Zol)
		\end{align}
		And analogously for \Prev and \Dp. The correctness of the induction step follows from the semantics of temporal operators and Lemma \ref{app:lm:loop}.
	\end{proof}
	
	Now, for $Q = \exists \Uol \p(\Xol, \Uol)$, we apply Lemma \ref{app:lm:rewriting} to $\p(\Xol, \Uol)$ and obtain $\psi_\varkappa(\Xol, \Uol, t)$. The required rewriting $\psi_Q(\Xol, t)$ is then the formula $\exists \Uol \psi_\varkappa(\Xol, \Uol, t)$.
\end{proof}

\PropinNL*
\begin{proof}
	The lower bound for \mdlnd follows from the analogous proof for temporal deductive databases \cite[Appendix A]{TDD-1988}, which uses only monadic IDBs and a `rigid' relation $\textit{Next}(i, j)$ holding at every moment of time. To simulate the latter in \mdlnd, assert $\textit{Next}(i, j)$ at time $0$ and refer to it by $\Dp \textit{Next}(i, j)$ at later moments. The upper bound, for the full \dlnd, is obtained by grounding rules of the program with the domain elements, obtaining a set of propositional \LTL formulas, converting the input database into a set of temporalised atoms (e.g. $R(\bar d) \in D_\ell$ becomes $\Next^\ell R(\bar d)$), and applying the polynomial-space satisfiability checking procedure for \LTL \cite{LTL-Decidability-and-Completeness}.
	
	The rest of the proof in dedicated to the case of \lmdlnd. The lower bound follows from \NLogSpace-hardness of query answering for plain linear datalog. To simplify the notation, we show the upper bound for the monadic queries only. The proof can be straightforwardly modified for the non-monadic case with $M_{\Pi, \Db}$ in Lemma~\ref{th:nl-bounds} changing to $2 (|\textit{IDB}(\pi)|(|\dom \Db|)^a)^2$, where $a$ is the maximal arity among $\textit{IDB}(\pi)$.
For a program $\Pi$, a rule
\begin{equation}\label{rule:for-expansions-rec1}
	C(X) \impd B(X, Y, U_1, \dots, U_m)\ \wedge\ \Next^k D(Y),
\end{equation}
($k \in \{-1, 0, 1\}$) with $C, D \in \textit{IDB}(\pi)$, data
instance $\Db$, $c, d \in \Delta_{\Db}$ and $\ell \in \mathbb Z$, we
write $C(c) \leftarrow^B_{\ell,k} D(d)$ if $\Db, \ell \models \exists
U_1, \dots, U_m (B(c, d, U_1, \dots, U_m))$. We write $C(c) \leftarrow_{\ell,k} D(d)$ if there exists a rule~\eqref{rule:for-expansions-rec1} in $\Pi$ such that $C(c) \leftarrow_{\ell,k}^B D(d)$.
Similarly, for an initialization rule
\begin{equation}\label{rule:for-expansions-init}
	C(X) \impd B(X, U_1, \dots, U_m),
\end{equation}
we write $C(c) \leftarrow^B_\ell$ if
$\Db, \ell \models \exists U_1, \dots, U_m (B(c, U_1, \dots, U_m))$,\\
and we define $C(c) \leftarrow_\ell$ analogously.
Firstly, we observe the following important \emph{monotonicity} property:
\begin{align}
  \notag C(c) \leftarrow_{\ell,k} D(d) &\text{ iff }C(c) \leftarrow_{\ell', k} D(d) \\
  \label{eq:mon} C(c) \leftarrow_\ell &\text{ iff }C(c) \leftarrow_{\ell'}
\end{align}
for any $c,d, C, D, k$ as above and any $\ell, \ell' \in \mathbb Z$ such that either $\ell, \ell' > \max \Db + K_\Pi$ or $\ell, \ell' < \min \Db - K_\Pi$, where $K_\Pi$ is the maximal number of nested temporal operators in a rule body of $\Pi$.

We call
\begin{equation}\label{eq:der0}
C_0(c_0) \leftarrow_{\ell_0, k_0} C_1(c_1) \leftarrow_{\ell_1, k_1} \dots \leftarrow_{\ell_{n-1}, k_{n-1}} C_n(c_n)
\end{equation}
an \emph{entailment sequence from } $C_0(c_0)$ \emph{at} $\ell_0$ \emph{to} $C_n(c_n)$ \emph{at} $\ell_n$ \emph{ in } $\Db$ if $C_i \in \textit{IDB}(\pi)$, $c_i \in \Delta_{\Db}$, $\ell_i \in \mathbb Z$, $\ell_{i+1} = \ell_{i}+k_i$. An entailment sequence is called \emph{initialized} if $C_n(c_n) \leftarrow_{\ell_n}$.
It should be clear that
\begin{multline*}
  \text{For all } \ell \in \Z \text{ and } g \in \dom(\Db), \Pi, \Db, \ell \models G(g) \text{ iff }\\
  \text{there exists some initialized entailment sequence from } G(g) \text{ at }\ell \text{ in }\Db.
\end{multline*}
Our aim now is to show:
\begin{lemma}\label{th:nl-bounds}
    If there exists an initialized entailment sequence from $C_0(c_0)$ at $\ell_0 \in [\min\Db, \max\Db]$ in \Db, then there exists an initialized entailment sequence $C_0(c_0)
    \leftarrow_{\ell_0, k_0} C_1(c_1) \leftarrow_{\ell_1, k_1} \dots
    \leftarrow_{\ell_{n-1}, k_{n-1}} C_n(c_n)$, with $\ell_i \in [\min \Db - K_\Pi - M_{\Pi, \Db}, \max \Db + K_\Pi + M_{\Pi, \Db}]$, where $M_{\Pi, \Db} = 2 (|\textit{IDB}(\pi)||\dom \Db|)^2$.
\end{lemma}
Because each $\ell_i$ is at most quadratic in $|\Db|$ it should be clear that there is an algorithm to check $\Pi, \Db, \ell \models G(g)$ that works in $\NLogSpace{}$ w.r.t. to $|\Db|$. In the remainder, we prove the above Lemma.
\begin{proof}
Consider a sequence~\eqref{eq:der0} in the assumption of the lemma and take the minimal $i$ such that:
\begin{itemize}
  \item either $\ell_i = \max \Db + K_\Pi+1$ and there exists $j > i$ with $\ell_j = \max \Db + K_\Pi+2(|\textit{IDB}(\pi)||\dom \Db|)^2+1$ and $\ell_t > \ell_i$ for all $i < t < j$,
  \item or $\ell_i = \min \Db - K_\Pi-1$ and there exists $j > i$ with $\ell_j = \min \Db - K_\Pi-2(|\textit{IDB}(\pi)||\dom \Db|)^2-1$ and $\ell_t < \ell_i$ for all $i < t < j$.
\end{itemize}
If such $i$ does not exist, then the sequence~\eqref{eq:der0} is already as required by the lemma and we are done. We show how to construct the required entailment sequence for the first case above; the second is analogous and left to the reader. There are two options:
\begin{description}
  \item[$(a)$] either for all $j > i$ we have $\ell_i > \max \Db + K_\Pi+1$,
  \item[$(b)$] or there is $j > i$ with $\ell_j = \max \Db + K_\Pi+1$.
\end{description}
First, we consider $(b)$. Take the minimal such $j$. We will show how to adjust~\eqref{eq:der0} by replacing the entailment sequence from $C_i(c_i)$ at $\ell_i$ to $C_j(c_j)$ at $\ell_j$ by another entailment sequence from $C_i(c_i)$ at $\ell_i$ to $C_j(c_j)$ at $\ell_j$ satisfying the condition of the lemma.
To this end, it will be convenient to consider \emph{derivation sequences} $C_0(c_0) \leftarrow_{k_0} C_1(c_1) \leftarrow_{k_1} \dots \leftarrow_{k_{n-1}} C_n(c_n)$ with $k_i \in \{-1, 0, 1\}$. We denote such a sequence by $\avec{der}$ and let $\textit{sh}_{\avec{der}}(0) = 0$ and $\textit{sh}_{\avec{der}}(i) = \sum_{j = 0}^{i-1} k_j$ for $0 < i \leq n$.
We denote by $\avec{der}(\ell_0)$ the entailment sequence $C_0(c_0) \leftarrow_{\ell_0, k_0} C_1(c_1) \leftarrow_{\ell_0+k_0, k_1} \dots \leftarrow_{\ell_0+k_0+\dots +k_{n-1}, k_n} C_n(c_n)$.
Let $\avec{der}$ be $C_i(c_i) \leftarrow_{k_i} \dots \leftarrow_{k_{j-1}} C_j(c_j)$ such that $\avec{der}(\avec{\ell}_i) = C_i(c_i) \leftarrow_{\ell_i, k_i} \dots \leftarrow_{\ell_{j-1}, k_{j-1}} C_j(c_j)$.
It follows from the minimality of $j$ that $\min\{ \textit{sh}_{\avec{der}}(i) \} = 0$ and $\textit{sh}_{\avec{der}}(j) = 0$.
It should be clear that if we are able to convert $\avec{der}$ to \begin{equation}\label{eq:derprime}\avec{der}' = C_0'(c_0') \leftarrow_{k_0'} C_1'(c_1') \leftarrow_{k_1'} \dots \leftarrow_{k_{m-1}} C_m'(c_m')
\end{equation}
such that
\begin{multline}\label{eq:der-conv}
  C_0'(c_0') = C_i(c_i)$, $C_m'(c_m')= C_j(c_j),\, \min\{ \textit{sh}_{\avec{der'}}(i) \} = 0, \\
  \textit{sh}_{\avec{der'}}(m) = 0, \text{ and }\max\{ \textit{sh}_{\avec{der'}}(i) \} \leq 2 (|\textit{IDB}(\pi)||\dom \Db|)^2,
\end{multline}
then we are done.
Indeed, due to~\eqref{eq:mon}, if $\avec{der}$ and $\avec{der}'$ are two derivation sequences with the same initial $C(c)$, the same final $D(d)$, and satisfying $\min\{ \textit{sh}(i) \} = 0$, then $\avec{der}'(\max \Db + K_\Pi+1)$ is an entailment sequence iff $\avec{der}(\max \Db + K_\Pi+1)$ is an entailment sequence.
In $\avec{der}$, consider the the smallest $i'$ such that $\textit{sh}_{\avec{der}}(i') = (|\textit{IDB}(\pi)||\dom \Db|)^2$ and then the smallest $j' > i'$ such that $\textit{sh}_{\avec{der}}(j') = (|\textit{IDB}(\pi)||\dom \Db|)^2$ (both $i'$ and $j'$ must exist by our assumptions). If $\avec{der}' = C_{i'}(c_{i'}) \leftarrow_{k_{i'}} C_{i'+1}(c_{i'+1}) \leftarrow_{k_{i'+1}} \dots C_{j'}(c_{j'})$ is such that $\max\{ \textit{sh}_{\avec{der}'}(i) \} \leq (|\textit{IDB}(\pi)||\dom \Db|)^2$, then we proceed to find the new $i' > j'$ and the new $j'$ such that $\textit{sh}_{\avec{der}}(i') = \textit{sh}_{\avec{der}}(j') = (|\textit{IDB}(\pi)||\dom \Db|)^2$ and treat the new $i',j'$ as the initial ones. Suppose, on the other hand, $\max\{ \textit{sh}_{\avec{der}'}(i) \} > (|\textit{IDB}(\pi)||\dom \Db|)^2$. It follows that there exists a pair $C(c), D(d)$, $i'', j''$ with $i' < i'' < j'' < j'$, and $i''', j'''$ with $i' < i''' < j''' < j'$, such that
\begin{itemize}
  \item $C(c) = C_{i''}(c_{i''}) = C_{i'''}(c_{i'''})$,
  \item $D(d) = C_{j''}(c_{j''}) = C_{j''}(c_{j'''})$,
  \item $\textit{sh}_{\avec{der'}}(i'') = \textit{sh}_{\avec{der'}}(j'') < \textit{sh}_{\avec{der'}}(i''') = \textit{sh}_{\avec{der'}}(j''') = t$
\end{itemize}
The $\avec{der}$ is illustrated in Fig. \ref{fig:proof-figure}.

\begin{figure}
	\centering
	\includegraphics[height=4cm, keepaspectratio]{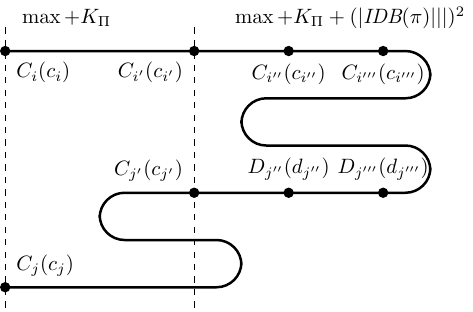}
	\caption{The derivation $\avec{der}$.}
	\label{fig:proof-figure}
\end{figure}

\noindent
It should be clear that $C(c), D(d)$, $i'', j''$ can be chosen so that $t \leq |\textit{IDB}(\pi)||\dom \Db|$.
We modify $\avec{der}$ by replacing it with $\avec{der}' = C_i(c_i) \leftarrow_{k_i} \dots C_{i''}(c_{i''}) \leftarrow_{k_{i'''}} C_{k_{i'''+1}}(c_{k_{i'''+1}}) \dots C_{j'''}(c_{j'''}) \leftarrow_{k_{j''}} C_{k_{j''+1}}(c_{k_{j''+1}}) \dots \leftarrow_{k_{j-1}} C_j(c_j)$. It should be clear that:
\begin{itemize}
  \item the length of $\avec{der'}$ is smaller than the length of $\avec{der}$ (which is equal to $j - i + 1$)
  \item $\min\{ \textit{sh}_{\avec{der'}}(i) \} = 0$
\end{itemize}
It is easy to see that $\avec{der}'$ will satisfy $(b)$ above again. We repeat the process described in the proof above treating $\avec{der}'$ as $\avec{der}$. By repeating this process finitely many times, we are able to achieve $\avec{der}'$ satisfying~\eqref{eq:der-conv}.

Now, consider the first option $(a)$ above. Let $\avec{der} = C_i(c_i) \leftarrow_{k_i} C_{i+1}(c_{i+1}) \leftarrow_{k_{i+1}} \dots C_n(c_n)$ such that $\avec{der}(\avec{\ell}_i) = C_i(c_i) \leftarrow_{\ell_i, k_i} \dots \leftarrow_{\ell_{n-1}, k_{n-1}} C_n(c_n)$. It follows $\min\{ \textit{sh}_{\avec{der}}(i) \} = 0$. As in the proof above, we need to convert $\avec{der}$ to $\avec{der}'$~\eqref{eq:derprime} satisfying
\begin{multline}\label{eq:derconvprime}
C_0'(c_0') = C_i(c_i)$, $C_m'(c_m')= C_n(c_n),\\
 \min\{ \textit{sh}_{\avec{der'}}(i) \} = 0,\, \text{ and }\max\{ \textit{sh}_{\avec{der'}}(i) \} \leq 2 (|\textit{IDB}(\pi)||\dom \Db|)^2.
\end{multline}
By our assumptions,~\eqref{eq:derconvprime} does not hold for $\avec{der}$, so consider the the smallest $i'$ in $\avec{der}$ such that $\textit{sh}_{\avec{der}}(i') = (|\textit{IDB}(\pi)||\dom \Db|)^2$. First, we assume that there exists no $j' > i'$ such that $\textit{sh}_{\avec{der}}(i') = \textit{sh}_{\avec{der}}(j')$. In this case, from $\avec{der}'' = C_{i'}(c_{i'}) \leftarrow_{k_{i'}} \dots \leftarrow_{k_{n}} C_n(c_n)$, we remove all the loops of the form $C(c) \leftarrow \dots \leftarrow C(c)$. It should be clear that $\max\{ \textit{sh}_{\avec{der}''}(i) \} \leq  |\textit{IDB}(\pi)||\dom \Db|$ and $|\min\{ \textit{sh}_{\avec{der}''}(i) \}| \leq  |\textit{IDB}(\pi)||\dom \Db|$. Therefore, $\avec{der}'$ obtained by removing the loops from $\avec{der}$ after $i'$ will satisfy~\eqref{eq:derconvprime} and we are done. If, on the other hand, there exists $j' > i'$ such that $\textit{sh}_{\avec{der}}(i') = \textit{sh}_{\avec{der}}(j')$, we adjust the sequence $C_{i'}(c_{i'}) \leftarrow_{\ell_{i'}, k_{i'}} \dots \leftarrow_{\ell_{j'-1}, k_{j'-1}} C_{j'}(c_{j'})$ in $\avec{der}$ as explained in the proof for $(b)$. The resulting $\avec{der}'$ will be shorter than $\avec{der}$ and, unless it satisfies~\eqref{eq:derconvprime}, $(a)$ will apply to it. It should be clear that after finitely many repetitions of the process described in the proof above (for $(a)$), we obtain $\avec{der}'$ satisfying~\eqref{eq:derconvprime}.

\end{proof}

\end{proof}

%%%%%%%%%%%%%%%%%%%%%%%%%%%%%%%%%%%%%%%%%%%%%%%%%%%%%%%%%%%%%%%%%%%%%%

% !TeX spellcheck = en_GB

\section{Proofs for Section \ref{sec:diamond-undecidability}} \label{app:sec:diamond-undecidability}

%%\setcounter{theorem}{7}

% If $M$ halts, $(\pi_M, G)$ is in \ACz. Otherwise, it is \NLogSpace-hard.
\LemmaHalting*
\begin{proof}
	(i) Suppose that $M$ halts and let~\eqref{def:machine:computation} be the (complete) computation (of length $m$). We will show that for any temporal database $\Db$ over $\Sigma$:
	\begin{multline}\label{eq:finite-exp}
		\Db, \pi_M, \ell \mdl G(d_0)\, \Leftrightarrow \text{ there exists } w(X) \in \expand(\pi, G) \\
		\text{with }\leq m+2 \text{ variables such that }\Db, \ell \mdl w(d_0)
	\end{multline}

	The direction $(\Leftarrow)$ is immediate. %by Lemma\nb{V: formulate}, 
	For the other direction, suppose $\Db, \pi_M, \ell \mdl G(d_0)$ for some $d_0 \in \obj(\Db)$ and a time $\ell$. By~\eqref{rule:machine:goal}, $\Db, \pi_M, \ell \mdl S_0(d_0)$ and $\Db, \ell \mdl U_1(d_0, d) \wedge U_2(d_0, d)$ for some $d \in \obj(\Db)$. Since $\Db, \pi_M, \ell \mdl S_0(d_0)$, three cases are possible:
	\begin{description}
		\item[$(i)$] $\Db, \pi_M, \ell \mdl \Drefl RV(d_0)$ by~\eqrefleft{rule:machine:representation-violation-first}. Then, $\Db, \ell'$, for some $\ell' \geq \ell$, satisfies the body of the rule~\eqrefright{rule:machine:representation-violation-first} or one of the rules~\eqref{rule:machine:representation-violation-last} with $X$ equal to $d_0$. Suppose $\ell' > \ell$ and $\Db, \ell'$ satisfies the body of~\eqrefright{rule:machine:representation-violation-first} (the other cases are analogous). Then, since $m \geq 0$, $w(X) = \Diamond (T(X, Y) \wedge \D T(X, Z)) \in \expand(\pi, G)$, and $\Db, \ell \mdl w(d_0)$, we are done.
		
		\item[$(ii)$] either $\Db, \pi_M, \ell \mdl \Drefl NE^{\varepsilon_1}_1(d)$ and $\Db, \ell \mdl \D^\alpha U_1(d_0, d)$ by~\eqrefleft{rule:machine:tv-state} or $\Db, \pi_M, \ell \mdl \Drefl NE^{\varepsilon_2}_2(d)$ and $\Db, \ell \mdl \D^\beta U_2(d_0, d)$ by~\eqrefright{rule:machine:tv-state}, for some transition $\Theta(s_0, \alpha, \beta) = (s_j, \varepsilon_1, \varepsilon_2)$. We consider the first case only (the second is analogous). Furthermore, we assume that $\alpha = +$ and $\varepsilon_1 = -1$ (the cases for $\alpha = 0$ and $\varepsilon_1 \in \{0, 1\}$ are analogous). Then, $\Db, \pi_M, \ell' \mdl NE^{\varepsilon_1}_1(d)$ for some $\ell' \geq \ell$ and $\Db, \ell'' \mdl U_1(d_0, d)$ for some $\ell'' > \ell$. We suppose $\ell'> \ell$, the alternative case is analogous. Because $\Db, \pi_M, \ell' \mdl NE^{\varepsilon_1}_1(d)$ and $\varepsilon_1 = -1$, $\Db, \ell'$ satisfies the body of~\eqrefleft{rule:machine:transition-violation2} or the next two rules in the list, with $Y$ equal to $d$. Suppose the body of~\eqrefleft{rule:machine:transition-violation2} is satisfied (the remaining two cases are analogous). Then, $\Db, \ell' \mdl U_1(d, d')$, for some $d' \in \obj(\Db)$. We take $w(X) = \Diamond U_1(X,Y) \land \Diamond (U_1(Y, Z) \land U_1(Y, X))$. As explained above, $\Db, \ell \mdl w(d_0)$. Clearly, the number of variables in $w$ is $\leq m+2$ and we are done.
		
		\item[$(iii)$] $\Db, \ell \mdl T(d_0, d_1) \wedge \D^\alpha U_1(d_0, d_1) \land \D^\beta U_2(d_0, d_1)$ and $\Db, \pi_M, \ell \mdl S_i(d_1)$ for some $d_1 \in \obj(\Db)$ and transition $\Theta(s_0, \alpha, \beta) = (s_i, \varepsilon_1, \varepsilon_2)$. If either $\Db, \pi_M, \ell \mdl \Drefl RV(d_0)$, or $\Db, \pi_M, \ell \mdl \Drefl NE^{\varepsilon_1}_1(d)$ and $\Db, \ell \mdl \D^\alpha U_1(d_0, d)$, or $\Db, \pi_M, \ell \mdl \Drefl NE^{\varepsilon_2}_2(d)$ and $\Db, \ell \mdl \D^\alpha U_2(d_0, d)$ for some $d \in \obj(\Db)$, then we are done as we can construct $w(X)$ such that $\Db, \ell \mdl w(d_0)$ as in the cases $(i)$ and $(ii)$. We assume neither of the above is the case. Observe that $\ell' \geq \ell$ and $\Db, \ell' \models U_k(d_0, d_1)$ implies $\ell' = \ell$, for $k \in \{1,2\}$. Indeed, recall that we have $\Db, \ell \models U_k(d_0, d)$ for some $d \in \obj(\Db)$. If $\Db, \ell' \models U_k(d_0, d_1)$ for $\ell' > \ell$, then $\Db, \pi_M, \ell \models RV(d_0)$ by~\eqref{rule:machine:representation-violation-last}, which is a contradiction. It follows that $\alpha = \beta = 0$. Therefore, $i = 1$. We also require for the future that $\ell' \geq \ell$ and $\Db, \ell' \mdl \exists x U_k(d_1, x)$ implies $\ell' = 0 + \varepsilon_i$. Indeed, for the sake of contradiction, suppose there is $d \in \obj(\Db)$ and $\ell' \geq \ell$ such that $\Db, \ell' \mdl U_1(d_1, d)$ and $\ell' \neq 0 + \varepsilon_1$ (the case of $k =2$ is analogous). Observe that in this case $\varepsilon_1 \neq -1$ because $\Theta(s_0, 0, 0) = (s_1, \varepsilon_1, \varepsilon_2)$ is a transition. Suppose, first, $\varepsilon_1 = 0$, then $\ell' > \ell$. Then, since $\Db, \ell \models U_1(d_0, d_1)$ and by~\eqrefright{rule:machine:transition-violation5}, we have $\Db, \pi_M, \ell \models NE^{\varepsilon_1}_1(d_1)$, which is a contradiction. Suppose, now that $\varepsilon_1 = 1$, then $\ell' = \ell$ or  $\ell' > \ell+1$. For $\ell' = \ell$ we use~\eqrefright{rule:machine:transition-violation6} and for $\ell' > \ell+1$ we use~\eqrefright{rule:machine:transition-violation-last} to conclude that $\Db, \pi_M, \ell \models NE^{\varepsilon_1}_1(d_1)$, which is a contradiction. We conclude for $i = 1$ that
		\begin{align}
			\notag&\Db, \pi_M, \ell \mdl S_i(d_i),\\
			\notag&\Db, \ell' \mdl \exists x U_1(d_i, x) \text{ implies }\ell' = \ell+a_i,\\ &\Db, \ell' \mdl \exists x U_2(d_i, x) \text{ implies }\ell' = \ell+ b_i,\label{eq:ih}
		\end{align}
		for all $\ell' \geq \ell$. Furthermore, we observe that there exists a partial expansion $w_i(X, X_i)$, containing (among others) variables $X$ and $X_i$, such that
		\begin{align}
			&w_i \text{ has }i+1 \text{ variables and }\Db, \ell \mdl w_i(d_0, d_i)\label{eq:ih1}
		\end{align}
		Indeed, for $i = 1$ we take $w_i = T(X,X_i) \land U_1(X,X_i) \land U_2(X,X_i)$. That $\Db, \ell \mdl w_1(d_0, d_1)$ follows from the proof above.
		
	\end{description}
	
	Suppose~\eqref{eq:ih} and~\eqref{eq:ih1} hold for $i<m$. We will show that either there is $w(X) \in \expand(\pi, G)$
	with $\leq m+2$  variables such that $\Db, \ell \mdl w(d_0)$, or that~\eqref{eq:ih} and~\eqref{eq:ih1} hold with $i+1$ in place of $i$. Since $\Db, \pi_M, \ell \mdl S_i(d_i)$, the following cases are possible (we observe that $\Db, \pi_M, \ell \mdl \Drefl RV(d_i)$ is not possible by~\eqref{eq:ih}):
	\begin{description}
		\item[$(i)$] either $\Db, \pi_M, \ell \mdl \Drefl NE^{\varepsilon_1}_1(d)$ and $\Db, \ell \mdl \D^\alpha U_1(d_i, d)$ by~\eqrefleft{rule:machine:tv-state} or $\Db, \pi_M, \ell \mdl \Drefl NE^{\varepsilon_2}_2(d)$ and $\Db, \ell \mdl \D^\beta U_2(d_i, d)$ by~\eqrefright{rule:machine:tv-state}, for some transition $\Theta(s_i, \alpha, \beta) = (s_j, \varepsilon_1, \varepsilon_2)$. Assuming that~\eqrefleft{rule:machine:tv-state} was the case (c.f. proof of $(ii)$ above), $\Db, \pi_M, \ell' \mdl NE^{\varepsilon_1}_1(d)$ for $\ell' > \ell$, $\alpha = +$, $\varepsilon_1 = -1$, and that $\Db, \ell'$ satisfies the body of \eqrefleft{rule:machine:transition-violation2}  with $Y$ equal to $d$ (the alternative cases are analogous), we obtain $w(X) = w_i(X, X_i) \land \Diamond U_1(X_i,Y) \land \Diamond (U_1(Y, Z) \land U_1(Y, X_i))$, for $Y,Z$ not occurring in $w_i$. By~\eqref{eq:ih1}, we conclude that $\Db, \ell \mdl w(d_0)$ and that the number of variables in $w$ is $\leq m+2$. We are done.
		\item[$(ii)$] $\Db, \ell \mdl T(d_i, d_{i+1}) \wedge \D^\alpha U_1(d_i, d_{i+1}) \land \D^\beta U_2(d_i, d_{i+1})$ and $\Db, \pi_M, \ell \mdl S_j(d_{i+1})$ for some $d_{i+1} \in \obj(\Db)$ and transition $\Theta(s_i, \alpha, \beta) = (s_j, \varepsilon_1, \varepsilon_2)$.  If either $\Db, \pi_M, \ell \mdl \Drefl NE^{\varepsilon_1}_1(d)$ and $\Db, \ell \mdl \D^\alpha U_1(d_i, d)$, or $\Db, \pi_M, \ell \mdl \Drefl NE^{\varepsilon_2}_2(d)$ and $\Db, \ell \mdl \D^\alpha U_2(d_i, d)$ for some $d \in \obj(\Db)$, then we are done as we can construct $w'(X)$ such that $\Db, \ell \mdl w'(d_i)$ as in the case $(i)$ above. We assume that neither of the above is the case. By~\eqref{eq:ih}, it follows that $\mathrm{sgn}(a_i) = \alpha$ and $\mathrm{sgn}(b_i) = \beta$, therefore $j = i+1$.
		We now require for the future that $\ell' \geq \ell$ and $\Db, \ell' \mdl \exists x U_1(d_{i+1}, x)$ implies $\ell' = \ell+a_i + \varepsilon_1$.
		Indeed, for the sake of contradiction, suppose there is $d \in \obj(\Db)$ and $\ell' \geq \ell$ such that $\Db, \ell' \mdl U_1(d_{i+1}, d)$ and $\ell' \neq \ell+a_i + \varepsilon_1$.
		Suppose, first, $\varepsilon_1 = -1$ (observe that by the construction of $M$, we have $a_i \geq 1$). It follows $\ell \leq \ell' < \ell + a_i - 1$ or $\ell' \geq \ell+a_i$. We consider the first option only (the second is left to the reader). Since $\Db, \ell + a_i \models U_1(d_i, d_{i+1})$, by~\eqref{rule:machine:transition-violation2} we obtain $\Db, \pi_M, \ell' \mdl NE^{\varepsilon_1}_1(d_{i+1})$ and then $\Db, \pi_M, \ell \mdl \Drefl NE^{\varepsilon_1}_1(d_{i+1})$. This is a contradiction. We leave the cases $\varepsilon_1 \in \{0, 1\}$ to the reader. Similarly it is shown that  $\ell' \geq \ell$ and $\Db, \ell' \mdl \exists x U_2(d_{i+1}, x)$ implies $\ell' = \ell+b_i + \varepsilon_2$. Recall that $a_{i+1} = b_i + \varepsilon_1$ and $b_{i+1} = b_i + \varepsilon_2$, so we have obtained~\eqref{eq:ih} for $i+1$ in place of $i$. Furthermore, suppose $\alpha = 0$ and $\beta = +$ (the alternative three cases are left to the reader). Then consider the partial expansion $w_{i+1}(X, X_{i+1}) = w_i(X, X_i) \land T(X_i,X_{i+1}) \land U_1(X_i,X_{i+1}) \land \D U_2(X_i,X_{i+1})$ for $X_{i+1}$ not occurring in $w_{i}$. Since $\Db, \ell \mdl w_i(d_0, d_i)$ by~\eqref{eq:ih1}, then $\Db, \ell \mdl w_{i+1}(d_0, d_{i+1})$ and we have also shown~\eqref{eq:ih1} for $i+1$ in place of $i$.
		
	\end{description}
	
	Suppose $i = m$. By $\Db, \pi_M, \ell \mdl S_i(d_i)$ and~\eqref{eq:ih}, the following two cases are possible in principle: $(i)$ above or $(ii)$ above. Suppose $(ii)$ is the case. By~\eqref{eq:ih}, it follows that $\mathrm{sgn}(a_i) = \alpha$ and $\mathrm{sgn}(b_i) = \beta$. So, there is a transition from $(s_m, \mathrm{sgn}(a_m), \mathrm{sgn}(b_m))$ to some $(s_j, \varepsilon_1, \varepsilon_2)$ which is a contradiction to the fact that $M$ halts after $m$ steps. So $(i)$ is the case and from the proof of $(i)$ we obtain $w(X) \in \expand(\pi, G)$ with $\leq m+2$ variables such that $\Db, \ell \mdl w(d_0)$.

	\ \\
	
	(ii) Suppose now that $M$ does not halt. We prove that $(\pi_M, G)$ is \NLogSpace-hard by a reduction from the directed reachability problem. Let $G = (V, E)$ be a directed graph with the set of nodes $V = \{v_0, \dots, v_n\}$, the start node $v_s$ and the target node $v_t$.
	Let the first $n$ steps of the computation of $M$ (c.f.~\eqref{def:machine:computation}) be $(s_0, a_0, b_0), \dots, (s_n, a_n, b_n)$.
	We construct a temporal database $\Db_G$ with $\obj(\Db) = V \times \{0, \dots, n\}$, such that $v_t$ is reachable from $v_s$ in $G$ iff $\Db, \pi_M, 0 \mdl G((v_s,0))$. Firstly, we add to $\Db_G$ the assertions $U_1((v_t,k), (v_t,k), 0), U_1((v_t,k), (v_t,k), 1)$ for each $0 \leq k \leq n$. Then, for each $(v_i, v_j) \in E$ and $0 \leq k \leq n-1$, we add to $\Db_G$ the assertions:
	\begin{itemize}
		\item $T((v_i,k), (v_j, k+1), 0)$
		\item $U_1((v_i,k), (v_j, k+1), a_k)$
		\item $U_2((v_i,k), (v_j, k+1), b_k)$.
	\end{itemize}
	We now show that $v_t$ is reachable from $v_s$ in $G$ iff $\Db_G, \pi_M, 0 \mdl G((v_s, 0))$.
	$(\Rightarrow)$ Let $v_{i_0}, \dots, v_{i_m}$ be a path in $G$ with $i_0 = s$ and $i_m = t$. It follows by the construction of $\Db_G$ and~\eqref{rule:machine:representation-violation-first},~\eqref{rule:machine:representation-violation-last} that $\Db_G, \pi_M, 0 \models S_m((v_{i_m}, m))$. Then, by~\eqref{rule:machine:state-transition} it follows that $\Db_G, \pi_M, 0 \models S_{m-1}((v_{i_{m-1}}, m-1))$. Further applying~\eqref{rule:machine:state-transition} we obtain $\Db_G, \pi_M, 0 \models S_{0}((v_{i_{0}}, 0))$ and finally by~\eqref{rule:machine:goal} we obtain $\Db_G, \pi_M, 0 \mdl G((v_s, 0))$. The direction $(\Leftarrow)$ is similar and left to the reader.
\end{proof}

%%%%%%%%%%%%%%%%%%%%%%%%%%%%%%%%%%%%%%%%%%%%%%%%%%%%%%%%%%%%%%%%%%%%%%

% !TeX spellcheck = en_GB

\section{Proofs for Section \ref{sec:automata}} \label{app:sec:automata}

We say that an automaton $\A = (S, \aleph, s^0, T, F)$, with the set of states $S$, input alphabet $\aleph$, initial state $s^0$, transition relation $T$, and the set of final states $F$, is $n$-\textit{manipulatable}, for $n \in \N$, if there is a polynomial $p(x)$ and one-to-one encodings $u \colon S \to \{0, 1\}^{p(n)}$ and $v \colon \aleph \to \{0, 1\}^{p(n)}$, such that:
\begin{enumerate}
	\item[]\textbf{(repr)} membership in each of the sets $u(S), \{u(s^0)\}$, $u(F)$, and $v(\aleph)$ is decidable in polynomial space;
	
	\item[]\textbf{(tran)} membership in the set $\{(u(s), v(\Lambda), u(s')) \mid (s, \Lambda, s') \in T\}$ is decidable in polynomial space.
\end{enumerate}
Note that it follows, from the fact that $u$ is one-to-one, that $|S| \leqslant 2^{p(n)}$.

%%\setcounter{theorem}{13}

% \begin{lemma}\label{app:lm:body-language:expand}
% 	The language $\expand(\pi, G)$ is recognised by a nondeterministic finite state automaton that is $|\pi|$-manipulatable.
% \end{lemma}
%
\LemmaExpand*
\begin{proof}
	We prove a slightly more general statement---that $\Expand(\pi, G)$ is recognised by a $|\pi|$-manipulatable nondeterministic finite state automaton.
	
	Assume that all input words are correct (this can be checked by a simple automaton of size $\Omc(1)$). To prove the lemma, we construct a two-way automaton $\A_{tw} = (S_{tw}, \Omega, s^0_{tw}, T_{tw}, F_{tw})$ that simulates the rules of the program~$\pi$, and then convert it into a one-way one. Essentially, the states of $\A_{tw}$ are the IDBs of $\pi$ (plus a final state) and the rules of $\pi$ become its transition rules. The initial state is the one associated with $G$. The behaviour of $\A_{tw}$ varies depending on whether it is currently reading a vertical or a horizontal segment of the input word. Being in a state `$D$' and seeing a letter $\Lambda$, the automaton can use a rule 
	\begin{equation*}
		D(X) \impd B(X, Y, \Ubf) \wedge \Next^k H(X),
	\end{equation*}
	if the body $B$ is in $\Lambda$, change its state to that of `$H$', and then move its head according to the following rules:
	\begin{itemize}
		\item On a vertical segment, move one step to the right. 
		\item On a horizontal segment, move by $|k|$ steps in the direction of the sign of $k$. When counting, the singleton $\{\top\}$ is ignored, and an attempt to go beyond the limits of the segment leads to a rejection. When reading the singleton symbol $\{\top\}$, the automaton can nondeterministically decide to go all the way to the beginning of the next vertical segment.
		\item Going from a vertical to a horizontal segment, the automaton stops at the letter containing $\bot$.
	\end{itemize} 
	The automaton comes to the final state by using one of the initialisation rules. 
	Such an automaton works over $\Omega$, but its state space is polynomial in the size of $\pi$. We convert it to a one-way nondeterministic automaton with an exponential blow-up, using the method of crossing sequences~\cite{Shepherdson-2way}. The states of the resulting automaton are relations of the form $\Rmc \sbs S_{tw} \times S_{tw}$.
	Thus, we have \textbf{(repr)}. For \textbf{(tran)}, we note that a composition of two such relations is easily computable, given that, for every $s, s' \in S_{tw}$ and $\Lambda \in \Omega$, we can check $(s, \Lambda, s') \in T_{tw}$ in \PSpace (that is, we have \textbf{(tran)} for the one-way automaton because we have a similar property for $\A_{tw}$, namely that $T_{tw}$ is decidable in \PSpace). 
\end{proof}

% \begin{lemma}\label{app:lm:body-language:notaccept}
%   The language $\notacceptom(\pi, G)$ is recognised by a nondeterministic finite state automaton which is $|\pi|$-manipulatable.
% \end{lemma}
%
\LemmaNoAccept*
\begin{proof}
	Again, we prove that $\notacceptom(\pi, G)$ is recognised by a $|\pi|$-manipulatable nondeterministic finite state automaton.
	
	Let \Db be a temporal database and $d \in \Delta_\Db$. We define $r(d)$ (and $-l(d)$) as the maximal (respectively, minimal) timestamp $\ell$ such that $\Db_\ell$ contains an atom involving $d$. Clearly, $r(d), -l(d) \in \tem(\Db)$ for all $d$. An $n$-\textit{cut model} of \Db and $\pi$, denoted $\E|_n$, is a temporal database over the joint EDB and IDB schema of $\pi$ obtained from a model $\E$ of \Db and $\pi$ by discarding, for every object $d$, all the (IDB) atoms with timestamp larger than $r(d) + n$ and smaller than $-l(d) - n$. Intuitively, $\E|_n$ preserves all information about $\E$ that is `reachable' from $\Db$ via a path of length $n$. The following lemma allows to focus on finite temporal databases instead of infinite models when dealing with certain answers.
	
\begin{lemma}\label{lm:body-language:cuts}
  Let $|\pi| = n$. Then $\Db, \pi, \ell \mdl G(d_1, \dots, d_k)$ if and only if $\E|_{n}, \ell \mdl G(d_1, \dots, d_k)$ for every $n$-cut mode of $\Db$ and $\pi$.
\end{lemma}
\begin{proof}
  The direction $(\Rightarrow)$ is trivial, since $\E|_{n}$ is obtained from a model of $\Db$ and $\pi$. For the other direction, let $\E$ be an arbitrary model of $\Db$ and $\pi$ and consider $\E|_{n}$. Since $\E|_{n}, \ell \mdl G(d_1, \dots, d_k)$, it must be that $\E, \ell \mdl G(d_1, \dots, d_k)$. Since $\E$ is arbitrary, we get $\Db, \pi, \ell \mdl G(d_1, \dots, d_k)$.
\end{proof}

Roughly speaking, the states of the automaton for $\notacceptom(\pi, G)$ will be formed by the $|\pi|$-cut models. We will encode them as words over an enriched version of $\Omega$. The definition is as follows. Given a rule body $B$ of $\pi$, its \textit{enrichment} is obtained by adding to $B$ various atoms of the form $\Next^k D(Y)$, where $D$ is an IDB, $Y$ is a variable of $B$ and $|k| \leqslant |\pi|$.
Moreover, we define $\Gamma_{\pi, e}$ as the set of all enrichments of rule bodies of $\pi$ and $\Omega_e = 2^{\Gamma_{\pi, e} \cup \{\bot, \top\}}$. An enrichment of a word $\alpha \in \Omega^*$ is a word $\alpha_e \in \Omega_e^*$ obtained by substituting each of rule bodies in $\alpha$ with one of its enrichments. Given an enriched word $\alpha_e$, we build a temporal database $\Db_{\alpha_e}$ just like we did it for non-enriched words. The schema of $\Db_{\alpha_e}$ is thus $\EDBpi \cup \IDBpi$.
We call $\alpha_e$ \textit{legal} if $\Db_{\alpha_e} = \E|_{n}$, where $n = |\pi|$, for some model $\E$ of $\pi$ and $\Db_\alpha$.
\begin{lemma}\label{lm:body-language:legality-checking}
  Given $\alpha_e$, its legality can be decided in \PSpace.
\end{lemma}
\begin{proof}
  Recall that the objects of $\Db_{\alpha_e}$ are the variables of $\alpha$. By the definition of enrichments, $\alpha_e$ provides information on IDB atoms at timestamps from $-l(Y) - |\pi|$ to $r(Y) + |\pi|$, for every $Y \in \Delta_{\Db_{\alpha_e}}$. Let $\Tline{Y}$, or the \textit{timeline} of $Y$, be the temporal database $\langle T^Y_\ell \mid -l(Y) - |\pi| \leqslant \ell \leqslant r(Y) + |\pi| \rangle$ where $T_\ell$ contains exactly the IDB atoms of $Y$ present in $\Db_{\alpha_e}$ at time $\ell$. We say that $\Tline{Y}$ can be extended to an \LTL-model under the rules of $\pi$ if there exists a model $\langle E_\ell^Y \mid \ell \in \Zbb \rangle$ of $\Tline{Y}$ and the set of horizontal rules of $\pi$, such that $E_\ell^Y$ coincides with $T_\ell^Y$ for $\ell \in \tem(T^Y)$.
		
  Our goal is to decide whether there exists a model \E of \Db and $\pi$ such that for $n = |\pi|$ we have $\Db_{\alpha_e} = \E|_{n}$. The critical observation is that the rules of $\pi$ whose bodies involve EDB atoms can only be used to infer IDB atoms at timestamps within the distance of at most $|\pi|$ from the borders of $\tem(\pi)$. Thus, the condition $\Db_{\alpha_e} = \E|_{n}$ is equivalent to the conjunction of the following two statements:
  \begin{enumerate}
  	\item $\Db_{\alpha_e}$ satisfies all rules of $\pi$;
  	\item for every $Y \in \Delta_{\Db_{\alpha_e}}$, its timeline $\Tline{Y}$ can be extended to an \LTL-model under the horizontal rules of $\pi$.
  \end{enumerate}
  A direct check of the first condition fits into \PSpace. For the second condition, we use the standard \LTL satisfiability checking algorithm \cite{LTL-Pnueli}, also working in polynomial space.
\end{proof}
	
Let us return to proving Lemma~\ref{lm:body-language:notaccept} and build the respective automaton $\A_{na}$ for $\notacceptom(\pi, G)$. A word $\alpha$ is in $\notacceptom(\pi, G)$ if either it is not correct, or $\Db_\alpha, \pi, 0 \not\mdl G(X)$ (recall that $X\in\Delta_{\Db_\alpha}$). The first condition can be checked by a simple automaton of size $\Omc(1)$. The second condition, by Lemma~\ref{lm:body-language:cuts}, means that there exists a $|\pi|$-cut model $\E|_{n}$ of $\Db_\alpha$ and $\pi$ such that $\E|_{n}, 0 \not\mdl G(X)$.
The idea, adapted from~\cite{Cosmadakis-Monadic-DLog-Boundedness-Decidable}, is to construct a nondeterministic automaton, that, given $\alpha$, verifies that it satisfies the second condition. Namely, it guesses an enrichment $\alpha_e$ and checks that: (i) it is legal, (ii) it does not contain $G(X)$ at time 0.
	
The difficult part is (i). By Lemma~\ref{lm:body-language:legality-checking}, the legality of $\alpha_e$ can be checked in polynomial space. However, we need to do it via a finite state automaton, i.e. in constant space. In doing this, we rely on the connectedness of $\pi$. Intuitively, $\alpha_e$ is legal if all timelines can be extended to \LTL-models and, moreover, every `neighbourhood' in $\Db_{\alpha_e}$ of `size' $|\pi|$ satisfies the rules of $\pi$. We now give a precise definition to the notion of `neighbourhood'. Let $\alpha \in \Omega^*$ be a correct word, and $\alpha_e = \Lambda_0\dots\Lambda_n$, $\Lambda_i \in \Omega_e$, be its enrichment. We define an undirected graph $G_{\alpha_e} = (V_{\alpha_e}, E_{\alpha_e})$ so that $V_{\alpha_e} = \{i \mid \Lambda_i \text{ letter of } \alpha_e\}$, and $E_{\alpha_e}$ is as follows. Respective nodes are connected by an edge for any two consecutive letters of the same vertical or horizontal segment. If $x_j$ is a vertical segment and $y_j$ is a horizontal segment that follows $x_j$, then the node of the last letter of $x_j$ is connected to that of the letter of $y_j$ that is marked by $\bot$. Similarly, the node of the letter of $y_j$ that is marked by $\top$ is connected with that of the first letter of the subsequent vertical segment $x_{j + 1}$. The \textit{distance} between two letters of $\alpha_e$ is that between the corresponding nodes in $G_{\alpha_e}$, and the $D$-\textit{neighbourhood} of the letter $\Lambda_i$ in $\alpha_e$, for $D \in \N$, is the word $N_D(i, \alpha_e)$ obtained from $\alpha_e$ by omitting all the letters that are at distance more than $D$ from $\Lambda_i$. We further define the \textit{left $D$-neighbourhood} of $\Lambda_i$ as its $D$-neighbourhood in the word $\Lambda_0\dots\Lambda_i$.
The key observation is the following, given the connectedness of $\pi$.
\begin{lemma}\label{lm:body-language:legality-by-neighbourhoods}
  An enrichment $\alpha_e$ is legal if and only if for every letter $\Lambda_i$ in $\alpha_e$, its neighbourhood $N_D(i, \alpha_e)$, $D = |\pi|$, is legal as an enriched word over $\Omega_e$.
\end{lemma}
	
We further observe that for every letter $\Lambda_i$, the neighbourhood $N_D(i, \alpha_e)$ is contained in the left $2D$-neighbour\-hood of the rightmost letter of $\alpha_e$ that belongs to $N_D(i, \alpha_e)$. Thus, reading $\alpha_e$ left to right, it is enough for $\A_{na}$ to ensure that every left $2D$-neighbourhood is legal for $D = |\pi|$. To do so, the automaton can remember in its state, at every moment, the left $2D$-neighbourhood of the current letter and reject once it encounters a non-legal one. When reading a vertical segment, it suffices to remember the left neighbourhood of the previous letter and update it accordingly for the current one. In the case of reading a horizontal segment, left to right, the left neighbourhood changes substantially once we pass over the symbol $\bot$. To handle this, $\A_{na}$ must know both the left neighbourhood of the previous letter, as well as that of the last letter of the preceding vertical segment. Both can be described by the respective words over $\Omega_e$.
	
It remains to show that $\A_{na}$ is $|\pi|$-manipulatable. Clearly, every enriched letter $\Lambda \in \Omega_e$ can be encoded by a binary sequence of polynomial length. The word of enriched letters that describes a left $2D$-neighbourhood, for $D = |\pi|$, is thus also representable in space $\poly{|\pi|}$. Then the property \textbf{(repr)} is given by Lemma \ref{lm:body-language:legality-checking}, and that of \textbf{(tran)} is given by construction of $\A_{na}$. This proves Lemma~\ref{lm:body-language:notaccept}.
\end{proof}

%%%%%%%%%%%%%%%%%%%%%%%%%%%%%%%%%%%%%%%%%%%%%%%%%%%%%%%%%%%%%%%%%%%%%%

% !TeX spellcheck = en_GB

\section{Proofs for Section \ref{sec:next}} \label{app:sec:next}

%If a query is vertically unbounded, then it is \LogSpace-hard.
\LmVerticallyUnboundedMeansLogSpaceHard*
\begin{proof}
  Since $\expandom(\pi, G) \sbs \acceptom(\pi, G)$, and $(\pi, G)$ is vertically unbounded, for any $k$ there is a word $\alpha \in \acceptom(\pi, G)$, $\height{\alpha} > k$, such that the longest prefix of $\alpha$ of height $k$ is in $\notacceptom(\pi, G)$.
  Let $\Apic$ be the deterministic automaton that recognizes $\acceptom(\pi, G)$ (Corollary~\ref{cor:body-language:acceptom}). By flipping its final states we obtain a deterministic automaton %$\Anpic$
  for $\notacceptom(\pi, G)$ with exactly the same set of states and transition relation. For any $k$, let $\beta_k, \gamma_k$ be the words such that $\height{\beta_k} = k, \height{\gamma} > 0$, the first segment of $\gamma_k$ is vertical, $\beta_k \in \notacceptom(\pi, G)$, and $\beta_k\gamma_k \in \acceptom(\pi, G)$. Furthermore, let $s_k^m$ be the state of $\Apic$ that is reached upon reading the $m$th vertical symbol of $\beta_k$. If $k$ is sufficiently larger than the number of states of \Apic, then there must be a repetition in the sequence $s_k^1, \dots, s_k^{k - 1}$. Let $i, j$ be such that $s_k^i = s_k^j$, and let $\xi$ be the part of $\beta_k$ from the beginning till the $i$-th vertical symbol, $\upsilon$ the part from that point till the $j$-th vertical symbol, and $\zeta$---the remaining part. Then these words have the following properties:
\begin{enumerate}
	\item $\height{\xi\upsilon} < k, \height{\upsilon} > 0$; 
	\item for any $i \geqslant 0$ we have $\xi\upsilon^i\zeta \in \notacceptom(\pi, G)$ and $\xi\upsilon^i\zeta\gamma \in \acceptom(\pi, G)$;
	\item words $\xi$, $\upsilon$, and $\zeta\gamma$ are correct.
\end{enumerate}
The latter property is ensured by the fact that the last symbols of $\xi$ and $\upsilon$ are vertical, so we never `cut' inside a horizontal segment.

For a correct word $\alpha$, $\Db_\alpha$ always contains the object $X$, the answer variable of the temporal CQ $\alpha(X)$. If $\alpha$ contains an IDB atom $D(Y)$, we denote that $Y$ by $Y_\alpha$ and, moreover, we write $\ell_\alpha$ for the timestamp at which $D'(Y_\alpha)$ would appear in $\Db_\alpha$ if we substituted $D$ with a fresh EDB predicate $D'$. Finally, let $\wid{\alpha} = |\tem(\Db_\alpha)|$.

Now we are ready to give a reduction from the undirected reachability problem. Let $\Gmc = (V, E)$ be an undirected graph and $s, p \in V$. We build a temporal database $\Db_\Gmc$ with $s_0 \in \Delta_{\Db_\Gmc}$, so that $\Db_\Gmc, \pi, 0 \mdl G(s_0)$ if and only if $s$ and $p$ are connected by a path in $\Gmc$. Assume w.l.o.g. that $(v, v) \in E$ for every $v \in V$ and let $V_i = \{v_i \mid v \in V\}$ be a copy of $V$ indexed by $i \in \Nbb$. The domain of $\Db_\Gmc$ will be the union of $V_i$ for $0 \leqslant i < |V|$, plus some auxiliary objects. First, add a copy of $\Db_\xi$ to $\Db_\Gmc$ so that $Y_\xi$ at time $\ell_\xi$ in $\Db_\xi$ is glued to $s_0$ at time $0$ in $\Db_\Gmc$. In the same fashion, for every $(u, v) \in E$ add a copy of $\Db_\upsilon$, glueing $X$ at 0 to $u_i$ at $i \cdot \wid{\upsilon}$ and $Y_\upsilon$ at $\ell_\upsilon$ to $v_{i + 1}$ at $(i + 1) \cdot \wid{\upsilon}$, for every $i$ from $0$ to $|V| - 2$. Do the same, connecting $v_i$ and $u_{i + 1}$. Finally, add a copy of $\Db_{\zeta\gamma}$ to $\Db_\Gmc$ so that $X$ at 0 of $\Db_{\zeta\gamma}$ is glued to $p_{|V| - 1}$ at time $(|V| - 1) \cdot \wid{\upsilon}$ of $\Db_{\Gmc}$.

It is not hard to see that if $s$ and $p$ are connected in $\Gmc$, then $\Db_\Gmc, \pi, 0 \mdl G(s_0)$. For the other direction, note that $\Db_\Gmc, \pi, 0 \mdl G(s_0)$ means there is $w \in \expand(\pi, G)$ and a homomorphism $h$ from $\Db_w$ to $\Db_\Gmc$ that maps $X$ to $s_0$ and $0$ to $0$. If there exists a variable $Z$ of $w$ that is mapped to $p_{|V| - 1}$ (at any moment of time), we are done, since $\pi$ is connected. Otherwise, if no variable is mapped to $p_{|V| - 1}$, and hence to any node of the copy of $\Db_{\zeta\gamma}$ attached to it, we note that, by construction, there exists an $m$ and a homomorphism $g$ from $\Db_w$ to $\Db_{\xi\upsilon^m}$ with $g(X) = s_0$, $g(0) = 0$. Thus, $\xi\upsilon^m$ would be in $\Accept(\pi, G)$, a contradiction.
%
%\noindent
%$(\Rightarrow)$ Let $\Db, \Db', \Db''$ be as in the definition of vertical unboundedness.
%Fix a $k > 0$ and set $m = k \cdot |\pi|$.
%Then, by the definition of the three pointed databases,
%$\Db' \circ \Db^m \circ \Db'', \pi, 0 \mdl C(X)$ and
%$\Db' \circ {\Db}^m, \pi, 0 \not\mdl C(X)$.
%Hence, there exists an expansion $w \in \expand(\pi, G)$ such that $\Db' \circ \Db^m \circ \Db'', 0 \mdl w(X)$.
%	
%By Lemma~\ref{lm:preliminaries:tcq-next-homomorphisms}, we have that there is a homomorphism $h$ from $\Db_w$ to $\Db' \circ \Db^m \circ \Db''$, but there is no homomorphism from $\Db_w$ to $\Db' \circ \Db^m$, for otherwise we would have $\Db' \circ \Db^m, 0 \mdl w(X)$. Let $w = uv$, such that $u$ is the longest prefix of $w$ such that $h$ is a homomorphism from $\Db_u$ to $\Db'\circ\Db^m$. Since $m = k \cdot |\pi|$, we conclude that $\height{u} \geqslant k$. Consequently, $[w] \in \expandom(\pi, G)$ and its prefix of height $k$ is in $\notacceptom(\pi, G)$.
%%
\end{proof}

% 	If $(\pi, G)$ is vertically bounded, then the data complexity of $(\pi, G)$ coincides with that of the decision problem for $\notacceptom(\pi, G)$, modulo reductions computable in \ACz.
\LmVerticallyBoundedMeansACz*
\begin{proof}
	We first give a reduction from the decision problem for $\acceptom(\pi, G)$ to answering $(\pi, G)$. Given a word $\alpha \in \Omega^*$, checking that it is correct can be done in \ACz. For a correct word, construct $\Db_\alpha$. Then, by definition, $\Db_\alpha, \pi, 0 \mdl G(X)$ if and only if $\alpha \in \acceptom(\pi, G)$.
	
	The reduction from answering $(\pi, G)$ to deciding the language $\acceptom(\pi, G)$ is more involved. Recall that $\Db, \pi, \ell \mdl G(d)$, for some $\Db, d, \ell$, if and only if there is $\alpha \in \expandom(\pi, G)$ and a homomorphism $h$ from $\Db_\alpha$ to $\Db$ with $h(X) = d$ and $h(0) = \ell$. Since $(\pi, G)$ is vertically bounded, by definition there exists a number $k$, uniformly for all expansions, such that the longest prefix of $\alpha$ of height $k$ is in $\acceptom(\pi, G)$.
	Then we arrive to the following characterisation: $\Db, \pi, \ell \mdl G(d)$ if and only if there exists a tuple $\langle x_1, x_2, \dots, x_n\rangle$ of words in $\Omega$ composed of vertical letters and such that $|x_1| + \dots + |x_n| \leqslant k$, and a tuple $\langle y_1, y_2, \dots, y_n\rangle$ of words composed of horizontal letters, so that the word $\alpha = x_1y_1\dots x_ny_n$ belongs to $\acceptom(\pi, G)$ and there is a homomorphism $\Db_\alpha \to \Db$ which maps $X$ to $d$ and $0$ to $\ell$. We observe that there is a finite number of the tuples of the form $\langle x_1, x_2, \dots, x_n\rangle$ with $|x_1| + \dots + |x_n| \leqslant k$. We further note that $\acceptom(\pi, G, k)$, by definition, is closed upwards under the relation $\preccurlyeq$: if $\alpha \preccurlyeq \beta$, then $\beta \in \acceptom(\pi, G, k)$. Thus, checking $\Db, \pi, \ell \mdl G(d)$ is reduced to a search in \Db for a tuple $\langle x_1, \dots, x_n\rangle$ of vertical words of joint height $k$ and a $\preccurlyeq$-maximal tuple $\langle y_1, y_2, \dots, y_n\rangle$ of horizontal words such that $(x_1y_1\dots x_ny_n) \in \acceptom(\pi, G)$. The search can be done in \ACz, so the complexity of the whole procedure is determined by that of the decision problem for $\acceptom(\pi, G)$.
\end{proof}

% Checking if a connected linear \MTDLog query with \Next or \Prev is vertically bounded can be done in \PSpace.
\LmVerticallyBoundedChecking*
\begin{proof}
	We take the automaton $\A_{na}$ for $\notacceptom(\pi, G)$, constructed in Lemma~\ref{lm:body-language:notaccept}, and the (one-way) automaton $\A_e$ for $\Expand(\pi, G)$ from Lemma~\ref{lm:body-language:expand}. Let $\A_{vb}$ be their Cartesian product. Then $\A_{vb}$ is $|\pi|$-manipulatable. The condition that we need to check on $\A_{vb}$ is the following:
	\begin{itemize}
		\item[] \begin{itemize}
			\item[$(*)$] For every $k$ there is a word $\beta \in \notacceptom(\pi, G)$, $\height{\beta} = k$, and a word $\gamma$, $\height{\gamma} > 0$, such that the first body of $\gamma$ is vertical and $\beta\gamma \in \expandom(\pi, G)$.
		\end{itemize}
	\end{itemize}
	In $\A_{vb}$, call a transition \textit{vertical} if it is performed upon reading a vertical body and \textit{horizontal} otherwise. Then $(*)$ is true if and only if there is a state $q$ of $\A_{vb}$ that is accepting for $\A_{na}$ and that has two features:
	\begin{enumerate}
		\item There is a path from the initial state of $\A_{vb}$ to $q$ that contains a cycle such that there is at least one vertical transition in that cycle.
		\item There is a state $q'$ and a vertical transition from $q$ to $q'$, and a path from $q'$ to an accepting state of $\A_e$.
	\end{enumerate}
	
	Due to $|\pi|$-manipulatability, looking for paths and cycles can be done on the fly in nondeterministic space polynomial in $|\pi|$. Since \NPSpace = \PSpace, we are done.
\end{proof}

% Deciding boundedness of a connected linear monadic query in pure datalog is \PSpace-hard.
\LmBoundednessPSpaceHard*
\begin{proof}
	Let $p$ be a polynomial and $M$ a Turing machine that, given an input $y$, works in space $p(|y|)$. Let $\Smc$ denote the set of states of $M$ and $\Gamma$---its tape alphabet. Let $x$ be a word in the input alphabet of $M$. In the computation of $M$ on $x$, encode every configuration by a word over the alphabet $\Gamma_0 = (\Gamma \cup (\Gamma \times \Smc))^{p(|x|)}$ in the standard way, and the computation itself as a concatenation of these words separated by the symbol $\# \notin \Gamma$.
	
	Introduce a unary EDB relation $Q_a$ for every symbol $a \in \Gamma_0^{\#} = \Gamma_0 \cup \{\#\}$, a unary EDB relation \textit{First} and a binary EDB relation \textit{Next} to encode words over $\Gamma_0^{\#}$. For example, the word $abc$ is encoded as $\textit{First}(d), Q_a(d), \textit{Next}(d, d'), Q_b(d'), \textit{Next}(d', d''), Q_c(d'')$.
	
	We define a connected linear monadic program $\pi_{M, x}$ with two IDBs $G$ (as usual, for \textit{Goal}) and $F$ (for \textit{Finger}), so that $(\pi, G)$ is bounded if and only if $M$ accepts $x$. We note that $(\pi_{M, x}, F)$ will be always unbounded and thus we are talking about query boundedness, but not program boundedness. 
	
	Assume without loss of generality that $M$ halts on $x$ if and only if it accepts. The idea is to let $F$ to start from the \textit{First} symbol in an encoding of $M$'s computation on $x$ and to slide along this encoding until a wrong transition is found. The idea is that all configurations have the same length $p(|x|)$. If $d$ represents a tape cell in the current configuration, we can access the object that represent the same cell in the next one by a path of length $p(|x|)$ along the relation \textit{Next}.
	
	First, we derive $G$ immediately if the first configuration is not the initial one, i.e. that having the word $x$ followed by $p(|x|) - |x|$ blank symbols, and then a $\#$, with the machine observing the first symbol of $x$:
	\begin{equation}\label{rule:pure:initial-configuration}
		G(X) \impd \textit{Next}(X, X_1) \wedge \textit{Next}(X_1, X_2) \wedge \dots \wedge \textit{Next}(X_{k-1}, X_k) \wedge Q_a(X_k),
	\end{equation}
	for every $k$ from $0$ to $p(|x|)$ and every $a \in \Gamma_0^{\#}$ which should not be in the $k$th place of the initial configuration.
	
	Having ensured that we are observing the initial configuration, we begin the slide:
	\begin{equation}\label{rule:pure:start}
		G(X) \impd \textit{First}(X) \wedge F(X).
	\end{equation}
	Recall that there exists a 6-ary relation $R$ such that if $u = a_1\dots a_{p(|x|)}$ is a configuration encoding and $v = b_1\dots b_{p(|x|)}, b_i \in \Gamma_0$, then $v$ is an encoding of the subsequent configuration of $u$ if and only if $(a_i, a_{a+1}, a_{i + 2}, b_i, b_{i + 1}, b_{i + 2}) \in R$, $1 \leqslant i \leqslant p(|x|) - 2$. Let $R^{\#} = R \cup \{\#\}\times \Gamma_0^2 \cup \Gamma_0^2 \times \{\#\}$. We infer $F(X)$ every time there is an outgoing path of \textit{Next} such that the three letters at $X$ and those at distance $p(|x|)$ on this path do not belong to $R^{\#}$:
	\begin{multline}\label{rule:pure:bad-transition}
			F(X) \impd \textit{Next}(X, X_1) \wedge \dots \wedge \textit{Next}(X_{p(|x|)+1}, X_{p(|x|)+2}), \\
			Q_a(X) \wedge Q_b(X_1) \wedge Q_c(X_2) \wedge Q_d(X_{p(|x|)}) \wedge Q_e(X_{p(|x|) + 1}) \wedge Q_f(X_{p(|x|) + 2}),
	\end{multline}
	for all tuples $(a, b, c, d, e, f) \notin R^{\#}$. Furthermore, we infer $F$ if there are more than one letter in a tape cell:
	\begin{equation}\label{rule:pure:two-letters}
		F(X) \impd Q_a(X) \wedge Q_b(X), \text{ for $a \neq b$.}
	\end{equation}
	Finally, we move $F$ forward along \textit{Next}, provided that the objects are labelled with some letters at all:
	\begin{equation}\label{rule:pure:shift}
		F(X) \impd Q_a(X) \wedge \textit{Next}(X, Y) \wedge Q_b(Y) \wedge F(Y),
	\end{equation}
	for all $a, b \in \Gamma_0^{\#}$.
	
	Given a database $D$ over the EDB schema $\Sigma = \{Q_a \mid a \in \Gamma_0^{\#}\} \cup \{\textit{First}, \textit{Next}\}$, we call a tuple $(d_0, \dots, d_{k - 1}) \in \Delta_D^k$ a \textit{Next-path} of length $k$ if $\textit{Next}(d_0, d_1)$, \dots, $\textit{Next}(d_{k - 2}, d_{k - 1}) \in D$. We say that it is \textit{labeled by} a word $u \in (\Gamma_0^{\#})^k$ if $u = a_0\dots a_{k - 1}$ and $Q_{a_i}(d_i) \in D$. Note that a \textit{Next}-path may be labeled by more than one word.
	
	Let $u_{\textit{init}}$ be the encoding of the initial configuration, $\pi_{M, x}$ be the set of rules defined in \eqref{rule:pure:initial-configuration}---\eqref{rule:pure:shift}, and $D$ a database over $\Sigma$ with $d_0 \in \Delta_D$. Suppose, $D, \pi \mdl G(d_0)$. There are two possible cases. The first is that, by rule \eqref{rule:pure:initial-configuration}, there exists a \textit{Next}-path of length at most $p(|x|) + 1$ starting at $d_0$ and labeled by a word $u$ which is \textit{not} a prefix of $u_{\textit{init}}\#$. Otherwise, $\textit{First}(d_0) \in D$ and $D, \pi \mdl F(d_0)$. In the latter case there is $w \in \expand(\pi, F)$ such that $D \mdl w(d_0)$. Let $w$ be the shortest expansion with this property. Then, from the form of the expansions of $(\pi, F)$ we conclude that no prefix of $w$ is in $\accept(\pi, F)$. Let $X, X_1, \dots, X_{k - 1}, X_k$ be the variables of $w$ in the order of appearance and $h$ be the homomorphism from $D_w$ to $D$, with $h(X) = d_0$ and $h(X_i) = d_i$, $1 \leqslant i \leqslant k - 1$. Then $(d_0, \dots, d_{k - 1})$ is a \textit{Next}-path in $D$ which is labeled by a unique word $u$. Let $u'$ be the prefix of $u$ ending at the last $\#$ in $u$. Then $u'$ is a prefix of an encoding of the full computation of $M$ on $x$. The fact that it is a prefix of such is ensured by that any \textit{Next}-path of length $p(|x|)$ starting in $d_0$ is labeled by an encoding of the initial configuration.
	
	Thus, if $M$ halts on $x$ there is only finitely many possible $u'$ and $(\pi, G)$ is bounded. Otherwise, it is unbounded.
\end{proof}

%For every linear connected \MTDLog query $(\pi, G)$ which uses only operators \Next and \Prev, there exist a pure datalog query $(\pi_d, G)$ and a pure \LTL query $(\pi_t, G)$, such that:
%%
%\begin{enumerate}
%	\item $(\pi, G)$ is vertically bounded if and only if $(\pi_d, G)$ is bounded;
%	\item if $(\pi, G)$ is vertically bounded then its data complexity coincides with that of $(\pi_t, G)$.
%\end{enumerate}
\PropDecomposition*
\begin{proof}
	Fix a linear connected query $(\pi, G)$. By Corollary \ref{cor:body-language:acceptom}, there is a deterministic finite state automaton, denoted $\Apig = (S, \Omega, s_0, \delta, F)$, that recognises the language $\acceptom(\pi, G)$. We introduce an IDB $S_i$ for each $s_i \in S$. Moreover, for every $\Lambda \in \Omega$ we have a unary EDB relation $\Lambda$ and an additional unary EDB \textit{End} to mark the end of the word. Then $\pi_t$ is a program composed of the following rules:
	\begin{align}\label{rule:pi_t}
		&G(X) \impd S_0(X), \text{ for the initial state } s_0\\
		&S_i(X) \impd \Lambda(X) \wedge \Next S_j(X), \text{ for every transition } \delta(s_i, \Lambda) = s_j\\
		&S_i(X) \impd \textit{End}(X), \text{ for every } s_i \in F\\
		&S_i(X) \impd \Lambda(X) \wedge \Lambda'(X), \text{ for every } s_i \in S \text{ and } \Lambda \neq \Lambda'.
	\end{align}
	By construction, $\Db, \pi_t, \ell \mdl G(d)$ if and only if there exists $\ell' \geqslant \ell$ such that:
	\begin{enumerate}
		\item for all $l, \ell \leqslant l < \ell'$, there is $\Lambda_l \in \Omega$ such that $\Lambda_l(d) \in D_l$;
		\item either there is $l, \ell \leqslant l < \ell'$, such that $\Lambda(d), \Lambda'(d) \in D_l$ for two different $\Lambda, \Lambda'$, or $\textit{End}(d) \in D_\ell$ and the word $\Lambda_\ell\Lambda_{\ell + 1} \dots \Lambda_{\ell' - 1}$ belongs to $\acceptom(\pi, G)$.
	\end{enumerate}
	The first property is checkable in \ACz, and the second coincides with the decision problem for $\acceptom(\pi, G)$.
	
	For $\pi_d$, we introduce a binary EDB $\Lambda(X, Y)$ for every vertical letter of $\Omega$, and skip the horizontal letters. The rules are as follows:
	\begin{align}\label{rule:pi_t}
		&G(X) \impd S_0(X), \text{ for the initial state } s_0\\
		&S_i(X) \impd \Lambda(X, Y) \wedge S_j(Y), \text{ for every transition } \delta(s_i, \Lambda) = s_j \text{ with a vertical letter } \Lambda\\
		&S_i(X) \impd \textit{True}, \text{ for every } s_i \in F\\
		&S_i(X) \impd S_j(Y), \text{ whenever } \delta(s_i, \alpha) = s_j
	\end{align}
	for every word $\alpha$ composed of horizontal letters.
	
	It remains to prove that $(\pi_d, G)$ is bounded whenever $(\pi, G)$ is vertically bounded. So assume that $(\pi, G)$ is vertically unbounded. Then for every $k$ there is $\alpha \in \expandom(\pi, G) \sbs \acceptom(\pi, G)$, such that $\height{\alpha} > k$ and if $\beta$ is a prefix of $\alpha$ of $\height{k}$ then $\beta \in \notacceptom(\pi, G)$. Let $\rho \in \expand(\pi_d, G)$ correspond to the run of \Apig on $\alpha$ and $\rho'$ be its prefix corresponding to that on $\beta$. By construction, $\rho' \in \notaccept(\pi_d, G)$, and, moreover, $|\rho| > k$ and $|\rho'| = k$. Since $k$ is arbitrary, $(\pi_d, G)$ is unbounded.
	
	For the other direction, let $(\pi_d, G)$ be unbounded and $\rho \in \expand(\pi_d, G)$ be such that its prefix of length $k$ is in $\notaccept(\pi_d, G)$. Take the $\alpha \in \acceptom(\pi, G)$ that corresponds to the run of $\Apig$ represented by $\rho$. Let $\beta$ be the longest prefix of $\alpha$ of height $k$ and $\rho'$ be the prefix of $\rho$ that corresponds to the run of \Apig on $\beta$. Then $|\rho'| = k$, and thus the related run is not accepting, hence $\beta \in \notacceptom(\pi, G)$. Thus, for every $k$ we found a word $\alpha_k \in \acceptom(\pi, G)$ of height greater than $k$, so that any prefix of $\alpha_k$ of height $k$ is in $\notacceptom(\pi, G)$. To finish the proof we need to show that such an $\alpha_k$ can be found in $\Expand(\pi, G)$. Since $\alpha_k \in \acceptom(\pi, G)$, there is $\gamma \in \Expand(\pi, G)$ and a homomorphism from $\Db_\gamma$ to $\Db_{\alpha_k}$ mapping $X$ to $X$ and 0 to 0. Take $k = md + 1$, where $m > 0$ and $d$ is the diameter of $\pi$. Then any $\gamma$ of height less then $m$ would be mapped fully if the prefix of $\alpha_k$ of height $k$. Since the latter is in $\notacceptom(\pi, G)$, the height of $\gamma$ is greater than $m$. By the same line reasoning, the prefix of $\gamma$ of height $m$ is itself in $\notacceptom(\pi, G)$. Since $m$ is arbitrary, we are done.
\end{proof}

%%%%%%%%%%%%%%%%%%%%%%%%%%%%%%%%%%%%%%%%%%%%%%%%%%%%%%%%%%%%%%%%%%%%%%

\end{document}